\documentclass[12pt,journal,final,letterpaper,onecolumn]{article}
\usepackage[top=1in, bottom=1in, left=1in, right=1in]{geometry}
\usepackage{natbib}

\usepackage{setspace}
\usepackage{microtype}
\usepackage{graphicx}
\usepackage{subfigure}
\usepackage{booktabs}
\usepackage{hyperref}

\usepackage{amssymb,amsfonts}

\usepackage{amsmath, amsthm}

\usepackage{hyperref}
\usepackage{amsfonts}
\usepackage{multirow}
\usepackage{makecell}

\usepackage{algorithm} 
\usepackage{algorithmic}  

\usepackage{amsmath}
\usepackage{mathtools}
\usepackage{amsfonts}

\usepackage{microtype}
\usepackage{graphicx}
\usepackage{booktabs} % for professional tables

\usepackage{caption}
\usepackage{pdfpages}
\usepackage{thmtools}
\usepackage{thm-restate}

\usepackage{stackengine}

\makeatletter
\newcommand{\distas}[1]{\mathbin{\overset{#1}{\kern\z@\sim}}}%

\newcommand{\beqs}{\vspace{0mm}\begin{eqnarray}}
\newcommand{\eeqs}{\vspace{0mm}\end{eqnarray}}
\newcommand{\barr}{\begin{array}}
\newcommand{\earr}{\end{array}}

\newcommand{\phiv}{\boldsymbol{\phi}}

\usepackage{booktabs}
\usepackage{siunitx}

\usepackage{color,soul}

\definecolor{alizarin}{rgb}{0.82, 0.1, 0.26}

\begin{document}
\title{Augment-Reinforce-Merge Policy Gradient\\ for Binary Stochastic Policy}
\author{Yunhao Tang$^*$, Mingzhang Yin$^\dagger$, and Mingyuan Zhou$^\dagger$\\$^*$Columbia University and $^\dagger$The University of Texas at Austin\\
 \texttt{\small yt2541@columbia.edu, mzyin@utexas.edu, mingyuan.zhou@mccombs.utexas.edu}}
 
\maketitle

\begin{abstract}
Due to the high variance of policy gradients, on-policy optimization algorithms are plagued with low sample efficiency. In this work, we propose Augment-Reinforce-Merge (ARM) policy gradient estimator as an unbiased low-variance alternative to previous baseline estimators on tasks with binary action space, inspired by the recent ARM gradient estimator for discrete random variable models \citep{yin2018arm}. We show that the ARM policy gradient estimator achieves variance reduction with theoretical guarantees, and leads to significantly more stable and faster convergence of policies parameterized by neural networks.
\end{abstract}

\section{Introduction}
There has been significant recent interest %and progress 
in
%Recent years have witnessed numerous successful applications of 
deep reinforcement learning (DRL) that combines reinforcement learning (RL) with powerful function approximators  such as neural networks, which leads to a wide variety of successful applications, ranging from board/video game playing %, video game playing 
to simulated/real life robotic control \citep{silver2016,mnih2013,schulman2015,levine2016}. One major area of DRL is on-policy optimization \citep{silver2016,schulman2015}, which progressively improves upon the current policy iterate until a local optima is found. 

As on-policy optimization flattens RL into a stochastic optimization problem, unbiased gradient estimation is carried out using REINFORCE or its more stable variants \citep{williams1992,mnih2016}. In general, however, on-policy gradient estimators suffer from high variance and need many more samples to construct high quality updates. Prior works have proposed variance reduction using variants of control variates \citep{gu2015muprop,gu2017interpolated,grathwohl2017backpropagation,wu2018variance}. %liu2018action
 However, recently \citet{tucker2018mirage} cast doubts on some aforementioned variance reduction techniques by showing that their implementation deviates from the proposed methods in the paper, which we will detail in the related work below.  In other cases, biased gradients are deliberately constructed to heuristically compute trust region policy updates \citep{schulman2017proximal}, which also achieve state-of-the-art performance on a wide range of tasks.
 %liu2018action
%yet recently \citet{tucker2018mirage} show that their implementations actually introduces bias to achieve the reported gains. 

In this work, we consider an unbiased policy gradient estimator based on the Augment-Reinforce-Merge (ARM) gradient estimator for binary latent variable models \citep{yin2018arm}.  We design a practical on-policy algorithm for RL tasks with binary action space, and show that the theoretical guarantee for variance reduction in \citet{yin2018arm} can be straightforwardly applied to the policy gradient setting. The proposed ARM policy gradient estimator is a plug-in alternative to REINFORCE gradient estimator and its variants \citep{mnih2016}, with minor algorithmic modifications. 

The remainder of the paper is organized as follows. In Section 2, we introduce background and related work on RL and  the ARM estimator for binary latent variable models. In Section 3, we describe the ARM policy gradient estimator,  including the derivation, theoretical guarantees, and on-policy optimization algorithm. In Section 4, we show via thorough experiments that our proposed estimator consistently outperforms stable baselines, such as the A2C \citep{mnih2016} and recently proposed RELAX \citep{grathwohl2017backpropagation} estimators.

\section{Background}
\subsection{Reinforcement Learning}
Consider a Markov decision process (MDP), where at time $t$ the agent is in state $s_t \in \mathcal{S}$, takes action $a_t \in \mathcal{A}$, transitions to next state $s_{t+1} \in \mathcal{S}$ according to $s_{t+1}\sim p(\cdot\,|\,a_t,s_t)$ and receives instant reward $r_t = r(s_t,a_t) \in \mathbb{R}$. A policy is a mapping $\pi : \mathcal{S} \mapsto \mathcal{P}(\mathcal{A})$ where $\mathcal{P}(\mathcal{A})$ is the space of distributions over the action space $\mathcal{A}$. The objective of RL is to search for a policy $\pi$ such that the expected discounted cumulative rewards $J(\pi)$ is maximized
\begin{align}
    J(\pi) = \mathbb{E}_{\pi}\left[\sum_{t=0}^{T-1}  r_t \gamma^t\right],
    \label{eq:rlobj}
\end{align}
where $\gamma \in (0,1]$ is a discount factor and $T$ is the horizon. Let $\pi^\ast = \arg\max_\pi J(\pi)$ be the optimal policy. For convenience, we define under policy $\pi$ the value function as $V^\pi(s) = \mathbb{E}[\sum_{t=0}^{T-1}  r_t \gamma^t \,|\,s_0 = s]$ and action-value function as $Q^\pi(s,a) = \mathbb{E}[\sum_{t=0}^{T-1}  r_t  \gamma^t \,|\,s_0 = s, a_0=a]$. We also define the advantage function as $A^\pi(s,a) = Q^\pi(s,a) - V^\pi(s,a)$. By construction, we have 
$\mathbb{E}_{a\sim \pi(\cdot\,|\,s)}[A^\pi(s,a)] = 0$, $i.e.$, the expected advantage $A^\pi(s,a)$ under policy $\pi$ is zero.

\subsection{On-Policy Optimization}
One way to find $\pi \approx \pi^\ast$ is through direct policy search. Consider parameterizing the policy $\pi_\theta$ with parameter $\theta \in \Theta$ where $\Theta$ is the space of parameters. If the policy class is expressive enough such that $\pi^\ast \in \{\pi_\theta\,|\, \theta \in \Theta\}$, we can recover $\pi^\ast$ with some parameter $\theta^\ast$. In on-policy optimization, we start with a random policy $\pi_\theta$ and iteratively improve the policy through gradient updates $\theta \leftarrow \theta + \alpha g_\theta$ for some learning rate $\alpha > 0$. We can compute the gradient
\begin{align}
g_\theta =  \mathbb{E}_\pi\left[\sum_{t=0}^{T-1} Q^\pi(s_t,a_t) \nabla_\theta \log \pi_\theta (a_t\,|\,s_t)\right].
\label{eq:reinforcedpg}
\end{align}
It is worth noting that $g_\theta$ is almost but not exactly the REINFORCE gradient of $J(\pi_\theta)$ \citep{williams1992}. The approximation $g_\theta \approx \nabla_\theta J(\pi_\theta)$ is due to the absence of discount factors $\gamma^t$ at each term in the summation of (\ref{eq:reinforcedpg}). In recent practice, $g_\theta$ is used instead of the exact REINFORCE gradient $\nabla_\theta J(\pi_\theta)$ since the factor $\gamma^t$ aggressively weighs down terms with large $t$ \citep{schulman2015,schulman2017,mnih2016}, which leads to poor empirical performance. In our subsequent derivations, we treat $g_\theta$ as the standard gradient and we let $\hat{g}_\theta = \sum_t Q^\pi(s_t,a_t) \nabla \log \pi_\theta(a_t\,|\,s_t)$ as an one-sample unbiased estimate of $g_\theta$ such that $\mathbb{E}[\hat{g}_\theta] = g_\theta$. However, $\hat{g}_\theta$ generally exhibits very high variance and does not entail stable updates. Actor-critic policy gradients \citep{mnih2016} subtract the original estimator  $\hat{g}_\theta$ by a state-dependent baseline function $b(s_t)$ as a control variate for variance reduction. A near-optimal choice of $b(s_t)$ is the value function $b(s_t) \approx V^\pi(s_t)$, which yields the following one-sample unbiased actor-critic gradient estimator
\begin{align}
    \hat{g}_\theta^{(\text{AC})} &= \sum_{t=0}^{T-1} (Q^\pi(s_t,a_t) - V^\pi(s_t)) \nabla_\theta \log \pi_\theta (a_t\,|\,s_t) \nonumber \\
    &= \sum_{t=0}^{T-1} A^\pi(s_t,a_t)  \nabla_\theta \log \pi_\theta (a_t\,|\,s_t),
    \label{eq:a2cpg}
\end{align}
where we still have $\mathbb{E}[\hat{g}_\theta^{(\text{AC})}] = g_\theta$. We also call (\ref{eq:a2cpg}) the A2C policy gradient estimator \citep{mnih2016}. In practice, the action-value function $Q^\pi(s_t,a_t)$ is estimated via Monte Carlo samples and the value function $V^\pi(s_t)$ is approximated by a parameterized critic $V_\phi(s_t) \approx V^\pi(s_t)$ with parameter $\phi$. 

\subsection{Augment-Reinforce-Merge Gradient for Binary Random Variable Models}
Discrete random variables are ubiquitous in probabilistic generative modeling. For the presentation of subsequent work, we limit our attention to the binary case. Let $z \sim \text{Bernoulli}(\sigma(\phiv))$ denote a binary random variable such that
\begin{align}
    p(z=1\,|\,\phiv) = \sigma(\phiv) = \frac{\exp(\phi)}{\exp(\phi)  + 1}
\end{align}
where $\phi$ is the logit of the Bernoulli probability parameter and $\sigma(\mathbf{\cdot})$ is the sigmoid function. For multi-dimensional distributions, we have a vector of $V \geq 2$ binary random variables $\mathbf{z}$ such that each component $\mathbf{z}_i$ follows an independent Bernoulli distribution $\text{Bernoulli}(\phiv_i)$, which is denoted as % we denote such distribution as 
$\mathbf{z} \sim \Pi_{i=1}^V\text{Bernoilli}(\phi_i)$. In general, we consider an expected optimization objective $f(\mathbf{z})$ of the vector $\mathbf{z}$ 
\begin{align}
    \max_{\phi_i,1\leq i\leq V} \mathbb{E}_{\mathbf{z} \sim \Pi_{i=1}^V\text{Bernoulli}(\phi_i)} [f(\mathbf{z})].
    \label{eq:discreteobj}
\end{align}
To optimize (\ref{eq:discreteobj}), we can construct a gradient estimator of $\phiv$ for iterative updates. Due to the discrete nature of variables $\mathbf{z}$, the REINFORCE gradient estimator is the naive baseline but its variance can be too high to be of practical use. The ARM (Augment-Reinforce-Merge) gradient estimator \citep{yin2018arm} provides the following alternative (See also \textbf{Theorem~1}in \citep{yin2018arm}) 

\begin{restatable}[ARM estimator for multivariate binary]{thm}{armbinary}
\label{thm:armbinary} For a vector of $V$ binary random variables $\mathbf{z}$, the gradient of 
\begin{align}
    \mathcal{E}_{{\phiv}} = \mathbb{E}_{\mathbf{z} \sim \Pi_{i=1}^V\emph{\text{Bernoulli}}(\phi_i)} [f(\mathbf{z})] \nonumber
\end{align}
with respect to ${\phiv}$, the logits of the Bernoulli distributions can be expressed as 
\begin{align}
    \nabla_{\phiv} \mathcal{E}(\phiv) = \mathbb{E}_{\mathbf{u} \sim \Pi_{i=1}^V \mathcal{U}(0,1)} \left[f_{\Delta}(\mathbf{u},\phiv)(\mathbf{u} - \frac{1}{2})\right],
    \label{eq:armbinary}
\end{align}
where $f_\Delta(\mathbf{u},\phiv) = f(\mathbf{z}_1) - f(\mathbf{z}_2), \mathbf{z}_1 = \mathbf{1}[\mathbf{u} > \sigma(-\phiv)], \mathbf{z}_2 = \mathbf{1}[\mathbf{u} < \sigma(\phiv)]$. Here $\mathbf{1}[\mathbf{u} > \sigma(-\phiv)]$ is a $V$-dimensional vector with the $i$th component to be $1[u_i > \sigma(-\phi_i)]$ where $1[\cdot]$ is the indicator function.
\end{restatable}

The ARM gradient estimator was originally derived through a sequence of steps in \citet{yin2018arm}: first, an AR estimator is derived from augmenting the variable space (A) and applying REINFORCE (R); then a final merge step (M) is applied to several AR estimators for variance reduction. It was shown that through the merge step, the resulting ARM estimator is equivalent to the original AR estimator combined with an optimal control variate subject to certain constraints, which leads to substantial variance reduction with theoretical guarantees. We refer the readers to \citet{yin2018arm} for more details.

\subsection{Training Stochastic Binary Network}
One important application of the ARM gradient estimator is training stochastic binary neural networks. Consider a binary latent variable model with $T$ stochastic hidden layers $\mathbf{b}_{0:T-1}$ conditional on observations $\mathbf{x}$, we construct their joint distribution as 
\begin{align}
    q_{w_{0:T-1}}(\mathbf{b}_{0:T-1}\,|\,\mathbf{x}) = q_{w_1}(b_1\,|\,\mathbf{x}) \Pi_{t=1}^{T-1} q_{w_{t+1}}(\mathbf{b}_{t+1}\,|\,\mathbf{b}_t),
    \label{eq:stochasticnet}
\end{align}
where $\mathbf{b}_{0:T-1}$ are binary random variables and the conditional distributions $q_{w_{t+1}}(\mathbf{b}_{t+1}\,|\,\mathbf{b}_t)$ are Bernoulli distributions parameterized by $w_{t+1}$. In general, we construct the following objective in the form of (\ref{eq:discreteobj})
\begin{align}
     \mathcal{E}(w_{0:T-1}) &= \mathbb{E}_{\mathbf{b}_{0:T-1} \sim q_{w_{0:T-1}}} [f(\mathbf{b}_{0:T-1})]
     \label{eq:binarynet}
\end{align}
for some function $f(\mathbf{b}_{0:T-1})$.
In the context of variational auto-encoder (VAE) \citep{kingma2013}, $\mathcal{E}(w_{0:T-1})$ is the evidence lower bound (ELBO) \citep{blei2017}. We would like to optimize $\mathcal{E}(w_{0:T-1})$ using gradients $\nabla_{w_{0:T-1}}\mathcal{E}(w_{0:T-1})$, and this is enabled by the following ARM back-propagation theorem. %The original 
Proposition 6 in \citet{yin2018arm} addresses the VAE model, but there is no loss of generality when considering a general function $f(\mathbf{b}_{0:T-1})$ as in the following theorem.

\begin{restatable}[ARM Backpropagation]{thm}{armbackprop}
\label{thm:armbackprop} For a stochastic binary network with $T$ binary hidden layers, let $\mathbf{b}_0  = \mathbf{x}$, construct the conditional distributions as
\begin{align}
    q_{w_t} (\mathbf{b}_{t}\,|\,\mathbf{b}_{t-1}) = \emph{\text{Bernoulli}}(\sigma(\mathcal{T}_{w_{t}}(\mathbf{b}_{t-1}))),
\end{align}
for some function $\mathcal{T}_{w_t}(\cdot)$. Then the gradient of $\mathcal{E}(w_{0:T-1})$ w.r.t. $w_{t}$ can be expressed as 
\begin{align}
    &\nabla_{w_t} \mathcal{E}(w_{0:T-1}) = \mathbb{E}_{q(\mathbf{b}_{1:{t-1}})}\bigg[\mathbb{E}_{\mathbf{u}_t \sim \mathcal{U}(0,1)}\bigg[ \nonumber \\ f_\Delta (\mathbf{u}_t, \nonumber  &\mathcal{T}_{w_t}(\mathbf{b}_{t-1}),\mathbf{b}_{0:T-1-1})\left(\mathbf{u}_t - \frac{1}{2}\right)\bigg]\nabla_{w_t} \mathcal{T}_{w_t}(\mathbf{b}_{t-1})\bigg],
    \end{align}
where 
$
f_{\Delta}(\mathbf{u}_t,\mathcal{T}_{w_t}(\mathbf{b}_{t-1}),\mathbf{b}_{0:T-1-1}) = \mathbb{E}_{\mathbf{b}_{t+0:T-1}\sim q(\mathbf{b}_{t+0:T-1}\,|\,\mathbf{b}_t),~ \mathbf{b}_t = \mathbf{1}[\mathbf{u}_t > \sigma(-\mathcal{T}_{w_t}(\mathbf{b}_{t-1}))] }[f(\mathbf{b}_{0:T-1})]- \mathbb{E}_{\mathbf{b}_{t+0:T-1}\sim q(\mathbf{b}_{t+0:T-1}\,|\,\mathbf{b}_t),~ \mathbf{b}_t = \mathbf{1}[\mathbf{u}_t < \sigma(\mathcal{T}_{w_t}(\mathbf{b}_{t-1}))] }[f(\mathbf{b}_{0:T-1})]
$, which can be estimated via a single Monte Carlo sample.
\end{restatable}

\subsection{Related Work}
\paragraph{On-Policy Optimization.} On-policy optimization is driven by policy gradients with function approximation \citep{sutton1999}. Due to the non-differentiable nature of RL, REINFORCE gradient estimator \citep{williams1992} is the default policy gradient estimator. In practice, REINFORCE gradient estimator has very high variance and the updates can become unstable. Recently, \citet{mnih2016} propose advantage actor critic (A2C), which reduces the variance of the policy gradient estimator using a value function critic. Further, \citet{schulman2015high} introduce generalized advantage estimation (GAE), a combination of multi-step return and value function critic to trade-off the bias and variance in the advantage function estimation, in order to compute lower-variance downstream policy gradients.

\paragraph{Variance Reduction for Stochastic gradient estimator.} Policy gradient estimators are special cases of the general stochastic gradient estimation of an objective function written as an expectation $\mathbb{E}_{\mathbf{z}}[J(\mathbf{z})]$. To address the typical high-variance issues of REINFORCE gradient estimator, prior works have proposed to add control variates (or baseline functions) for variance reduction \citep{paisley2012variational,ranganath2014black,gu2015muprop,kucukelbir2016automatic}. Re-parameterization trick greatly reduces the variance when variables $\mathbf{z}$ are continuous and the underlying distribution is re-parametrizable \citep{kingma2013}. When variables $\mathbf{z}$ are discrete, \citet{maddison2016concrete,jang2016categorical} introduce a biased yet low-variance gradient estimator based on continuous relaxation of the discrete variables. More recently, REBAR \citep{tucker2017rebar} and RELAX \citep{grathwohl2017backpropagation} construct unbiased gradient estimators by using baseline functions derived from continuous relaxation of the discrete variables, whose parameters need to be estimated online. \citet{yin2018arm} propose an unbiased estimate motivated as a self-control baseline, and display substantial gains over prior works when $\mathbf{z}$ are binary variables. In this work, we borrow ideas from \cite{yin2018arm} and extend the ARM gradient estimator for binary stochastic network into ARM policy gradient for RL.

\paragraph{Variance Reduction for Policy Gradients.} By default, the baseline function for REINFORCE on-policy gradient estimator is only state dependent, and the value function critic is typically applied. \citet{gu2015muprop} propose Taylor expansions of the value functions as the baseline and construct an unbiased gradient estimator. \citet{grathwohl2017backpropagation,wu2018variance} 
%liu2018action
propose carefully designed action-dependent baselines, which can 
construct unbiased gradient estimator while in theory achieving more substantial variance reduction. Despite their reported success, \citet{tucker2018mirage} observe that subtle implementation decisions cause their code to diverge from the unbiased methods presented in the paper: \textbf{(1)} \citet{gu2015muprop,gu2017interpolated} achieve reported gains potentially by introducing bias into their advantage estimates. When such bias is removed, they do not outperform baseline methods; \textbf{(2)} \citet{liu2018action} achieve reported gains over state-dependent baselines potentially because the baselines are poorly trained. When properly trained, state-dependent baselines achieve similar results as the proposed action-dependent baselines; \textbf{(3)} \citet{grathwohl2017backpropagation} achieve gains potentially due to a bug that leads to different advantage estimators for their proposed RELAX estimator and the baseline. When this bug is fixed, they do not achieve significant gains. In this work, we propose ARM policy gradient estimator, a plug-in alternative which is unbiased and consistently outperforms A2C and RELAX for tasks with binary action space.

\section{Augment-Reinforce-Merge Policy Gradient}

Below we present the derivation of Augment-Reinforce-Merge (ARM) policy gradient. The high level idea is that we draw connections between RL and stochastic binary network - for a RL problem of horizon $T$, we interpret the action sequence $a_{0:T-1}$ as the stochastic $\mathbf{b}_{0:T-1}$ and derive the gradient estimator similarly as in Theorem \ref{thm:armbackprop}. We consider RL problems with binary action space $\mathcal{A} = \{0,1\}$.

\subsection{Time-Dependent Policy}
To make full analogy to stochastic binary networks with $T$ layers, we consider a RL problem with horizon $T$. We make the policy $\pi$ time-dependent by specifying a different policy $\pi_{\theta_t}$ for time $t$ with parameter $\theta_t$. We define a similar RL objective as (\ref{eq:rlobj})
\begin{align}
        J(\{\pi_{\theta_t}\}) = \mathbb{E}_{a_t \sim \pi_{\theta_t}(\cdot\,|\,s_t)}\left[\sum_{t=0}^{T-1}  r_t \gamma^t\right],
        \label{eq:timerlobj}
\end{align}
where we jointly optimize over all $\{\theta_t\},0\leq t\leq T-1$. To make the connection between (\ref{eq:timerlobj}) and (\ref{eq:binarynet}) explicit, we observe the following: we can interpret (\ref{eq:timerlobj}) as a special form of (\ref{eq:binarynet}) by setting the binary hidden variables as the actions $\mathbf{b}_t \coloneqq a_t$ and the observation $\mathbf{x}$ as the initial state $s_0$. The conditional distribution can be defined as $q_{w_t}(\mathbf{b}_t\,|\,\mathbf{b}_{t-1}) = q_{w_t}(a_t\,|\,a_{t-1}) \coloneqq \pi_{\theta_t}(a_t\,|\,s_t) p(s_t\,|\,a_{t-1},s_{t-1})$, which consists of the policy $\pi_{\theta_t}(a_t\,|\,s_t)$ and the transition dynamics $p(s_t\,|\,a_{t-1},s_{t-1})$. Finally we define the objective function $f(\mathbf{b}_{0:T-1}) = f(a_{0:T-1}) \coloneqq \mathbb{E}_{s_t \sim p(\cdot\,|\,a_t,s_{t-1})}[\sum_{t=0}^{T-1} r(s_t,a_t) \gamma^t]$, which depends only directly on $a_{0:T-1}$ (after marginalizing out the states $s_{0:T-1}$).

Since the action space is binary, we introduce a policy parameterization similar to stochastic binary network, $i.e.$,  $\pi(\cdot\,|\,s_t) = \text{Bernoulli}(\sigma(\mathcal{T}_{\theta_t}(s_t)))$. For any given $t = t^\ast$, consider the ARM estimator of (\ref{eq:timerlobj}) w.r.t. $\theta_{t^\ast}$ according to Theorem \ref{thm:armbackprop}. When we sample $a_t \sim \pi_{\theta_t}(\cdot\,|\,s_t)$, we can sample a uniform random variable $u_t \sim \mathcal{U}(0,1)$ then set $a_t = 1[u_t < \sigma(\mathcal{T}_{\theta_t}(s_t))]$. We define the \emph{pseudo action} as $a_t^{(s)} = 1[u_t >  \sigma(-\mathcal{T}_{\theta_t}(s_t))]$. By converting all variables in the binary stochastic network example into their RL counterparts as described above, we can derive the gradient of (\ref{eq:timerlobj}) w.r.t. $\theta_t$ for any given $t = t^\ast$
\begin{align}
    \nabla_{\theta_{t^\ast}} J(\{\pi_{\theta_t}\}) =  \mathbb{E}_{u_t^\ast \sim U(0,1)}[&(Q^{\{\pi_t\}}(s_t,a_{t^\ast}^{(s)}) - Q^{\{\pi_t\}}(s_t,a_{t^\ast})) (u_t^\ast - \frac{1}{2})]  \nabla_{\theta_{t^\ast}} \mathcal{T}_{\theta_{t^\ast}}(s_{t^\ast}),
    \label{eq:armgrad}
\end{align}
where $Q^{\{\pi_t\}}(s,a)$ is defined as $Q^{\{\pi_t\}}(s,a) = \mathbb{E}_{a_t \sim \pi_{\theta_t}(\cdot\,|\,s_t)}[\sum_{t=0}^{T-1} r_t\gamma^t \,|\,s_0=s,a_0=a]$, $i.e.$,  the expected cumulative rewards obtained by first executing $a_0 = a$ at $s_0 = s$ and following the time-dependent policy thereafter. Notice that this is not exactly the same as the action-value function we defined in Section 2.1 since the policy is time-dependent.

\subsection{Stationary Policy}
In practice we are interested in a stationary policy, which is invariant over time $a_t \sim \pi_\theta(\cdot\,|\,s_t)$. Since we have derived the gradient estimator for a time-dependent policy $\{\pi_{\theta_t}\}$, the most naive approach would be to share weights by letting $\theta_{t} \coloneqq \theta$ and linearly combining the per-step gradients (\ref{eq:armgrad}) across time steps. Since now the policy is stationary, we define $g_\theta^{(\text{ARM})}(t)$ as the stationary version of (\ref{eq:armgrad}) where  $Q^{\{\pi_t\}}$ is replaced by $Q^\pi$
\begin{align}
    g_\theta^{(\text{ARM})}(t^\ast) = \mathbb{E}_{u_t^\ast \sim U(0,1)}[&(Q^\pi(s_t,a_{t^\ast}^{(s)}) - Q^\pi(s_t,a_{t^\ast})) \nonumber \\ &(u_t^\ast - \frac{1}{2})]  \nabla_{\theta_{t^\ast}} \mathcal{T}_{\theta_{t^\ast}}(s_{t^\ast})].
    \label{eq:stationaryarmgrad}
\end{align}
The combined gradient is 
\begin{align}
    g_\theta^{(\text{ARM})} \coloneqq \sum_{t=0}^{T-1}  g_\theta^{(\text{ARM})}(t)
    \label{eq:fullgrad}
\end{align}
We denote $\hat{g}_\theta^{(\text{ARM})}(t)$ and $\hat{g}_\theta^{(\text{ARM})}$ the unbiased sample estimates of $g_\theta^{(\text{ARM})}(t)$ and $g_\theta^{(\text{ARM})}$ respectively. Importantly, we can show that $\hat{g}_\theta^{\text{(ARM)}}$ is unbiased, $i.e.$,  $\mathbb{E}[\hat{g}_\theta^{\text{(ARM)}}] =  g_\theta$ where $g_\theta$ is the standard gradient defined in Section 2.2. We summarize the result in the following theorem.

\begin{restatable}[Unbiased ARM policy gradient]{thm}{armpg}
\label{thm:armpg} When the ARM policy gradient is constructed as in (\ref{eq:fullgrad}), it is unbiased w.r.t. the true  gradient $g_\theta$ of the RL objective
\begin{align}
    \mathbb{E}[\hat{g}_\theta^{\text{(ARM)}}] = g_\theta\nonumber
\end{align}
\end{restatable}

\begin{proof}
Recall the standard gradient $g_\theta$ in (\ref{eq:reinforcedpg}), we define $
    g_\theta(t) \coloneqq  \mathbb{E}_{\pi_\theta}[ Q^{\pi_\theta}(s_t,a_t) \nabla_\theta \log \pi(a_t\,|\,s_t)]$. Now we show that by combining ARM gradients across all time steps $t$ (which correspond to all layers of a stochastic binary network), we can compute an unbiased policy gradient.
Recall the gradient at $t$ is computed based on (\ref{eq:stationaryarmgrad}), we explicitly compute the gradient in the  following. For simplicity, we remove all dependencies on $t$ and denote $\phi = \mathcal{T}_{\theta}(s_t)$, $\pi_1 = \pi(a_t=1\,|\,s_t) = \frac{\exp(\phi)}{1+\exp(\phi)}, ~\pi_0 = 1 - \pi_1$. Also the advantage function $A_1 = A^\pi(a_t=1,s_t), A_0 = A^\pi(a_t=0,s_t)$. Assume also $\pi_1 \geq \pi_0$ without loss of generality,
\begin{align}
    (\ref{eq:stationaryarmgrad}) &= -\int_{u > \pi_1} \frac{A_1}{\pi_1} (u - \frac{1}{2}) \nabla_\theta \phi du  \nonumber\\ &\ \ \ \ -\int_{u < \pi_0} \frac{A_0}{\pi_0} (u-\frac{1}{2}) \nabla_\theta \phi du \nonumber \\
    &= -(\frac{1}{2}A_0\pi_0 - \frac{1}{2} A_1\pi_1) \nabla_\theta \phi \nonumber \\
    &= A_1\pi_1 \nabla_\theta \phi
\end{align}
We also explicitly write down the standard gradient $g_\theta(t)$ at time step $t$ with the same notation
\begin{align}
    g_\theta(t) &= A_1 \nabla_\theta \pi_1 + A_0 \nabla_\theta \pi_0 = A_1 \pi_1 \nabla_\theta \phi. \nonumber
\end{align}
We see that conditional on the same state $s_t$, the ARM policy gradient sample estimate $\hat{g}_\theta^{(\text{ARM})}(t)$ at time step $t$ has its expectation equal to the standard gradient $g_\theta(t)$ at time $t$. To complete the proof, we marginalize out $s_t$ via the visitation probability under current policy $\pi_\theta$ and sum over time steps: $\mathbb{E}[\hat{g}_\theta^{(\text{ARM})}] = \mathbb{E}[\sum_t \hat{g}_\theta^{(\text{ARM})}(t)] = \sum_t g_\theta(t) = g_\theta$.
\end{proof}

\subsection{Variance Reduction for ARM Policy Gradient}
Here we compare the variance of the ARM policy gradient estimator $\hat{g}_\theta^{(\text{ARM})}$ against the REINFORCE gradient estimator $\hat{g}_\theta$. Analyzing the variance of the full gradient is very challenging, since we need to account for covariance between gradient components across time steps. We settle for analyzing the variance of each time component $\hat{g}_\theta^{(\text{ARM})}(t)$ and $g_\theta(t)$ for any $t \geq 0$. Similar analysis has been applied in \citet{fujita2018clipped}.

For simple analysis, we assume that all action-value functions (or advantage functions) can be obtained exactly, $i.e.$,  we do not consider additional bias and variance introduced by advantage function estimations. Conditional on a given $s_t$, we can show with results from \citet{yin2018arm} that
\begin{align}
    \frac{\sup_{\theta} \mathbb{V}[\hat{g}_\theta^{(\text{ARM})}(t)\,|\,s_t]}{\sup_{\theta} \mathbb{V}[\hat{g}_\theta(t)\,|\,s_t]} \leq \frac{16}{25}.
    \label{eq:statevariance}
\end{align}
Since we have established $\mathbb{E}[\hat{g}_\theta^{(\text{ARM})}(t)\,|\,s_t] = \mathbb{E}[\hat{g}_\theta(t)\,|\,s_t]$, we can show via the variance decomposition formula $\mathbb{V}[g] = \mathbb{E}[\mathbb{V}[g\,|\,s]] + \mathbb{V}[\mathbb{E}[g\,|\,s]]$ that for $\forall t\geq 0$
\begin{align}
    \sup_{\theta} \mathbb{V}[\hat{g}_\theta^{(\text{ARM})}(t)] \leq \sup_{\theta} \mathbb{V}[\hat{g}_\theta(t)].
    \label{eq:fullvariance}
\end{align}
We cannot obtain a tighter bound without making further assumptions on the relative scale of variance $\mathbb{V}[\hat{g}_\theta(t)\,|\,s_t]$ and the expectations $\mathbb{E}[\hat{g}_\theta(t)\,|\,s_t]$. Still, we can show that the ARM policy gradient estimator can achieve potentially much smaller variance than the standard gradient estimator, which entails faster policy optimization in practice.

We provide an intuition for why ARM policy gradient estimator can achieve substantial variance reduction and stabilize training in practice. In Figure \ref{figure:actiondiff}, we show the running average of $|a_t - a_t^{(s)}|$, $i.e.$,  the difference between on-policy actions and pseudo actions as training progresses. We show two settings of advantage estimations (left: A2C, right: GAE) which will be detailed in Section 4. We see that as training progresses, both actions are increasingly less likely to be different, often yielding $a_t = a_t^{(s)}$. In such cases the ARM policy gradient estimator achieves exact zero value  $\hat{g}_\theta^{(\text{ARM})} = 0$. On the contrary, prior baseline gradient estimators (\ref{eq:reinforcedpg},\ref{eq:a2cpg}) cannot take up exact zero values and will cause parameters to stumble around due to noisy estimates. We provide more detailed discussions in the Appendix.

\begin{figure}
\centering
\subfigure[\textbf{A2C}]{\includegraphics[width=.45\linewidth]{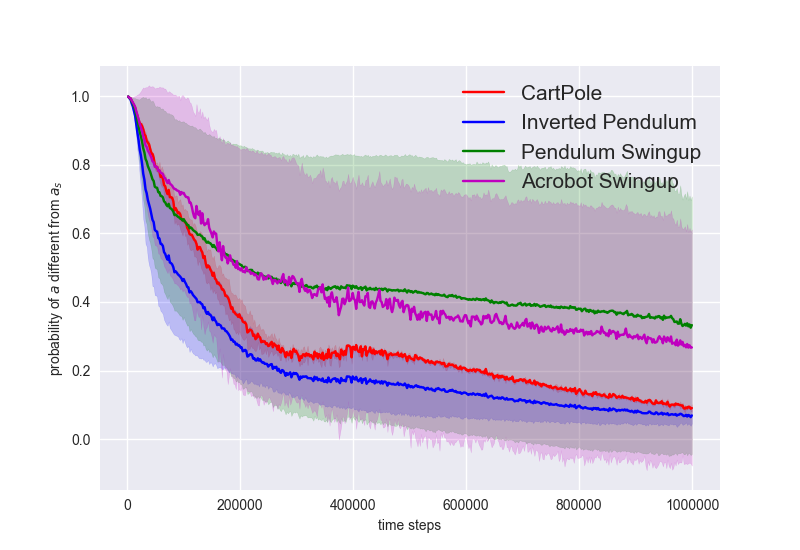}}
\subfigure[\textbf{GAE }]{\includegraphics[width=.45\linewidth]{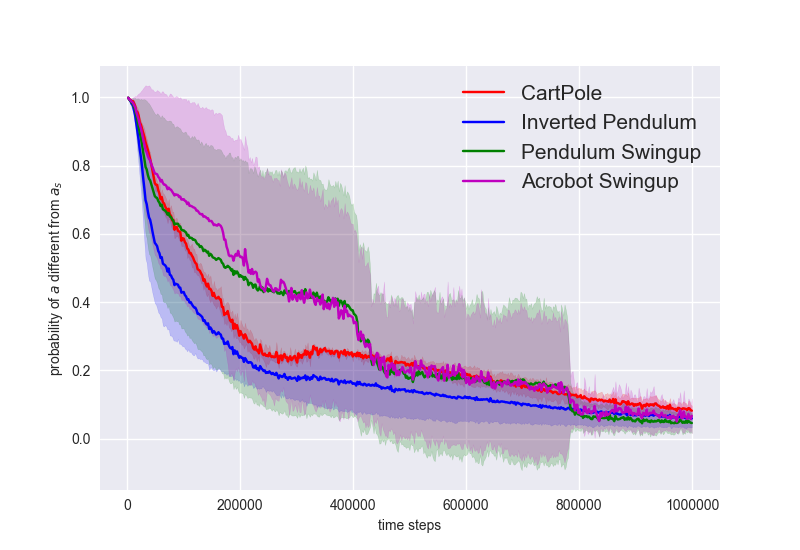}}
\caption{\small{Illustration of the frequency of $a_t^{(s)}$ different from $a_t$ as training progresses. The x-axis show the time steps in training, y-axis show the running average of $|a_t^{(s)} - a_t|$, a measure of how often they differ. Each curve shows the result for a different environment and we show the settings both for A2C (left) and GAE (right). As training progresses, on-policy actions $a_t$ are increasingly less likely to differ from the pseudo action $a_t^{(s)}$.}}
\label{figure:actiondiff}
\end{figure}

\subsection{Algorithm}
Here we present a practical algorithm that applies the ARM policy gradient estimator. The on-policy optimization procedure should alternate between collecting on-policy samples using current policy  and computing gradient estimator using these on-policy samples for parameter updates \citep{schulman2015,schulman2017}. 

We see that a primary difficulty with computing (\ref{eq:stationaryarmgrad}) is that it requires  the difference of two action-value functions $Q^\pi(s_t,a_t^{(s)}) - Q^\pi(s_t,a_t)$. Unless the difference is estimated by additional parameterized critics, in typical on-policy algorithm implementation, we only have access to the Monte Carlo estimators of action-value functions corresponding to on-policy actions. To overcome this, we estimate the difference by using the property  $\mathbb{E}_{a \sim \pi(\cdot\,|\,s)}[A^\pi(s,a)] = 0$. To be specific, when $a_t^{(s)} \neq a_t$, the advantage function of the pseudo action can be expressed as 
\begin{align}
A^{\pi}(s_t,a_{t}^{(s)}) = - \frac{\pi(a_t\,|\,s_t)}{1 - \pi(a_t\,|\,s_t)} A^{\pi}(s_t,a_{t})
\end{align}

Since the difference of action-value functions is identical to the difference of advantage functions, we have in general

\begin{align}
    Q^{\pi}(s_t,a_{t}^{(s)}) - %&
    Q^{\pi}(s_t,a_{t}) %= \nonumber \\ &
    =-(\frac{\pi(a_t\,|\,s_t)}{1 - \pi(a_t\,|\,s_t)} + 1)   A^\pi(s_t,a_t) |a_{t}^{(s)} -a_t|
    \label{eq:adv}
\end{align}
 We can hence estimate the difference in (\ref{eq:stationaryarmgrad}) with only on-policy advantage estimates $\hat{A}^\pi(s_t,a_t) \approx A^\pi(s_t,a_t)$, along with sampled pseudo action $a_t^{(s)}$ and on-policy action $a_t$.  Notice when $a_t^{(s)} = a_t$, the difference $Q^{\pi}(s_t,a_{t}^{(s)}) - Q^{\pi}(s_t,a_{t}) = 0$, therefore the gradient is zero with high frequency when the algorithm approaches  convergence, which leads to faster convergence speed and better stability.  

We design the on-policy optimization procedure as follows. At any iteration, we have policy $\pi_{\theta}$ with parameter $\theta$. We generate on-policy rollout at time $t$ by sampling a uniform random variable $u_t \sim \mathcal{U}(0,1)$, then construct  on-policy action $a_t = 1[u_t \leq \sigma(\mathcal{T}_\theta(s_t))]$ and pseudo action $a_t^{(s)} = 1[u_t \geq \sigma(-\mathcal{T}_\theta(s_t))]]$. Only the on-policy actions $a_t$ are executed in the environment, which returns  instant rewards $r_t$. We estimate the advantage functions $A^{\pi_\theta}(s_t,a_t)$ from on-policy samples using techniques from A2C \citep{mnih2016} or GAE \citep{schulman2015high}, and replace $A^\pi(s_t,a_t)$ in (\ref{eq:stationaryarmgrad}) by estimates $\hat{A}^\pi(s_t,a_t)$. The details of the advantage estimators are provided in the Appendix. Finally, the difference of the action-value functions can be computed using (\ref{eq:adv}) based purely on on-policy samples. With all the above components, we compute the ARM policy gradient (\ref{eq:stationaryarmgrad}) to update the policy parameters. The main algorithm is summarized in \textbf{Algorithm 1} in the Appendix.

\paragraph{ARM policy gradient as a plug-in alternative.} We note here that despite minor differences in the algorithm ($e.g.$,  need to sample pseudo actions $a_t^{(s)}$ along with $a_t$), ARM policy gradient estimator is a convenient plug-in alternative to other baseline on-policy gradient estimators (\ref{eq:reinforcedpg},\ref{eq:a2cpg}). All the other components of standard on-policy optimization algorithms ($e.g.$,  value function baselines) remain the same.

\section{Experiments}

We aim to address the following questions through the experiments: \textbf{(1)}  Does the proposed ARM policy gradient estimator outperform previous policy gradient estimators on binary action benchmark tasks? \textbf{(2)} How sensitive are these policy gradient estimators to hyperparameters, $e.g.$,  the size of the sample batch size? 

To address \textbf{(1)}, we compare the ARM policy gradient estimator with A2C gradient estimator \citep{mnih2016} and RELAX gradient estimator  \citep{grathwohl2017backpropagation}. We aim to study how advantage estimators affect the quality of the downstream gradient estimators: with A2C advantage estimation, we compare $\{\text{ARM},\text{A2C},\text{RELAX}\}$ gradient estimators; with GAE, we compare $\{\text{ARM},\text{A2C}\}$ gradient estimators \footnote{Here for A2C gradient estimator, we just replace the original A2C advantage estimators by GAE estimators and keep all other algorithmic components the same.}. We evaluate on-policy optimization with various gradient estimators on benchmark tasks with binary action space: for each policy gradient estimator, we train the policy for a fixed number of time steps ($10^6$) and across 5 random seeds. Each curve in the plots below shows the mean performance with shaded area as the standard deviation performance across seeds. In Figures \ref{figure:a2c} and \ref{figure:gae}, x-axis show the number of time steps at training time and y-axis show the performance. The results are reported in Section 4.1. To address \textbf{(2)}, we vary the batch size $\mathcal{B}$ to assess the effects of the batch size on the variance of the policy gradient estimators. We also evaluate the estimators' sensitivities to learning rates. The results are reported in Sections 4.2 and 4.3.

\paragraph{Benchmark Environments.} We focus on benchmark environments with binary action space illustrated in Figure \ref{figure:tasks}: All tasks are simulated by OpenAI gym and DeepMind control suite \citep{todorov2008,brockman2016,tassa2018deepmind}. Some tasks have binary action space by default. In cases where the action space is a real interval, $e.g.$,  $\mathcal{A} = [-1,1]$ for Pendulum, we binarize the action space to be $\mathcal{A} = \{-1,1\}$.

%In CartPole, the task consists of balancing a pole by applying forces to a cart which is constrained to move along an horizontal track. We have various versions of CartPole which correspond to different difficulty levels: higher versions imply more difficulty. For CartPole and Mountaincar, the tasks have binary action space by design. In Inverted Pendulum, Acrobot Swingup and Pendulum Swingup, the original task has a continuous action space $\mathcal{A} = [-1,1]$, we binarize the task by constraining the agent to only take $a \in \mathcal{A} = \{-1,1\}$.

\begin{figure*}[t]
\centering
\subfigure[\textbf{CartPole }]{\includegraphics[keepaspectratio,width=.19\linewidth]{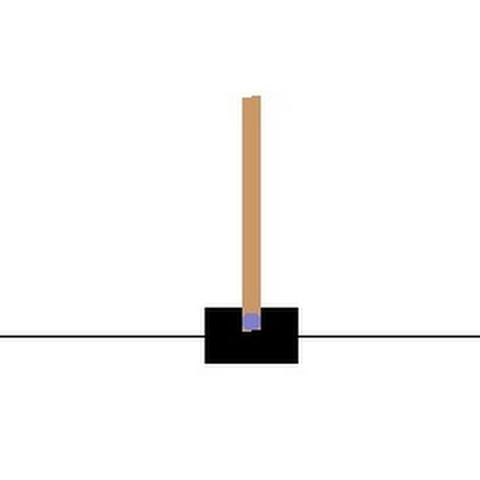}}
\subfigure[\textbf{MountainCar }]{\includegraphics[keepaspectratio,width=.19\linewidth]{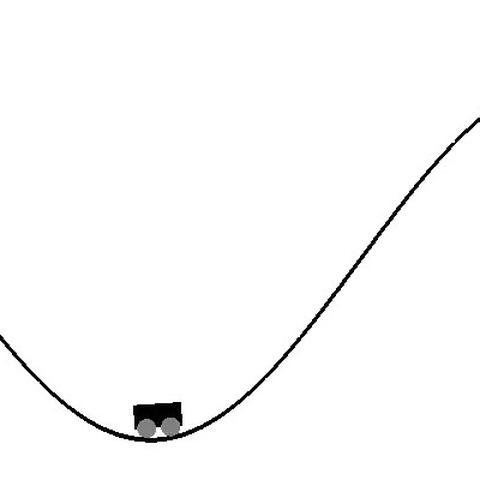}}
\subfigure[\textbf{Inverted Pendulum }]{\includegraphics[width=.19\linewidth]{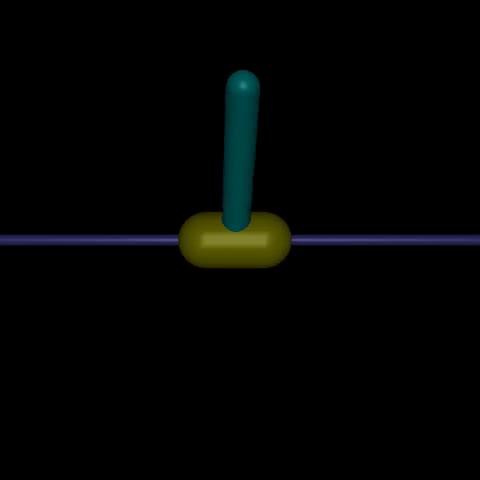}}
\subfigure[\textbf{Acrobot Swingup }]{\includegraphics[keepaspectratio,width=.19\linewidth]{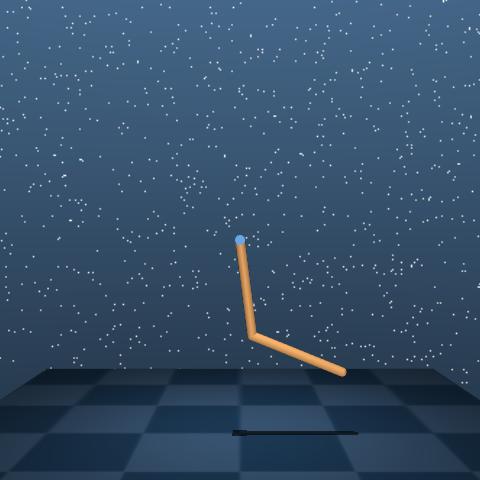}}
\subfigure[\textbf{Pendulum Swingup }]{\includegraphics[keepaspectratio,width=.19\linewidth]{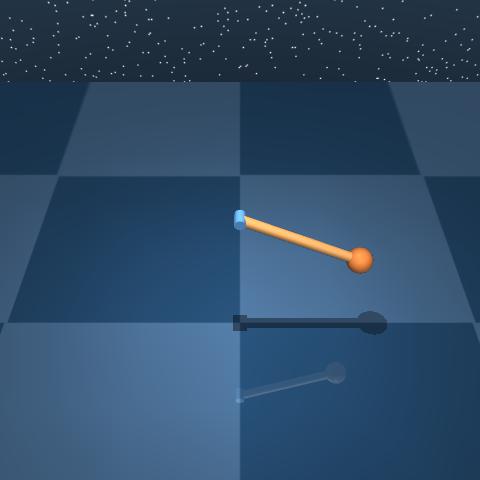}}
\caption{\small{Illustration of Benchmark tasks. Benchmark tasks (a)-(c) are from OpenaAI gym \citep{brockman2016} and (d)-(e) are from DeepMind Control Suite \citep{tassa2018deepmind}. These tasks have binary action space by design ((a)) or the action space is by design continuous $\mathcal{A} = [-1,1]$ and binarized for our setting ((b)-(e)).}}
\label{figure:tasks}
\end{figure*}

\paragraph{Implementations and Hyper-parameters.} All implementations are in Tensorflow \citep{abadi2016tensorflow} and RL algorithms are based on OpenAI baselines \citep{baselines} \footnote{\url{https://github.com/openai/baselines}}.
 All policies are optimized with Adam using best learning rates selected from $\in \{3\cdot 10^{-4},3\cdot 10^{-5}\}$. We refer to the original code of RELAX\footnote{\url{https://github.com/wgrathwohl/BackpropThroughTheVoidRL}} 
 but notice potential % errors 
 issues in their original implementation. We discuss such potential issues in the Appendix. We implement our own version of RELAX \citep{grathwohl2017backpropagation} by modifying the OpenAI baselines 
  \citep{baselines}. Recall that on-policy optimization algorithms alternate between collecting batch samples of size $\mathcal{B}$ and then compute gradient estimators based on these samples; here we set $\mathcal{B} = 2048$ by default.

%For the RELAX baseline we reference the code released by the original work \citep{grathwohl2017backpropagation} but implement our own

\paragraph{Policy Architectures.} We parameterize the logit function $\mathcal{T}_\theta(s)$ as a two-layer neural network with state $s$ as input and $64$ hidden units per layer. Each layer applies $\text{ReLU}$ non-linearity as the activation function. The output is a logit scalar $\mathcal{T}_\theta(s)$ where $\theta$ consists of weight matrices and biases in the neural network. For variance reduction we also parameterize a value function baseline with two hidden layers each with $64$ units, with $\text{relu}$ non-linear activation. Both the logit function and the value function have linear activation for the output layer. For RELAX gradient estimator, we use two  parameterized baseline functions with two layers, each with $64$ hidden units.

\subsection{Benchmark Comparison}
Here we compare the ARM policy gradient estimator with two baseline methods: A2C gradient estimator \citep{mnih2016} and the unbiased RELAX gradient estimator \citep{grathwohl2017backpropagation}. %,tucker2018mirage}. 
We note some critical implementation details: All gradient estimators require advantage estimation $\hat{A}(s,a) \approx A^\pi(s,a)$, we do not normalize the advantage estimates before computing the policy gradients as commonly implemented in \citet{baselines}. As observed in \citet{tucker2018mirage}, such normalization biases the original gradients for variance reduction, especially for action dependent baselines such as RELAX \citep{grathwohl2017backpropagation}.
Since our focus is on unbiased gradient estimators, we remove such normalization \citep{baselines} \footnote{Since GAE applies a combination of Monte Carlo returns and value function critics, the advantage estimator is still slightly biased to achieve small variance.}. 

\paragraph{A2C Advantage Estimator.} A2C Advantage estimators are simple combinations of Monte Carlo sampled returns and baseline functions. For $\{\text{ARM},\text{A2C}\}$ the baseline is the value function critic, while for RELAX the baseline is parameterized and trained to minimize the square norm of the gradients computed on mini-batches. Let $\hat{g}_\theta^{\text{RELAX}}$ be the RELAX gradient estimator and recall $\hat{g}_\theta^{\text{A2C}}$ as the A2C gradient estimator. Define a generalized RELAX gradient as $\hat{g}_\theta^{\text{RELAX}}(\tau) = \tau \cdot \hat{g}_\theta^{\text{A2C}}  + (1-\tau) \cdot \hat{g}_\theta^{\text{RELAX}} $ for $\tau \in [0,1]$. When $\tau = 1$ we recover an A2C gradient estimator (which is still different from our original A2C gradient estimator since the baseline functions are parameterized and trained differently). When $\tau = 0$ we recover the original RELAX estimator. In practice we find that $\tau > 0$ tends to significantly outperform the pure RELAX estimator.

\begin{figure}[!h!]
\centering
\subfigure[\textbf{MountainCar }]{\includegraphics[width=.4\linewidth]{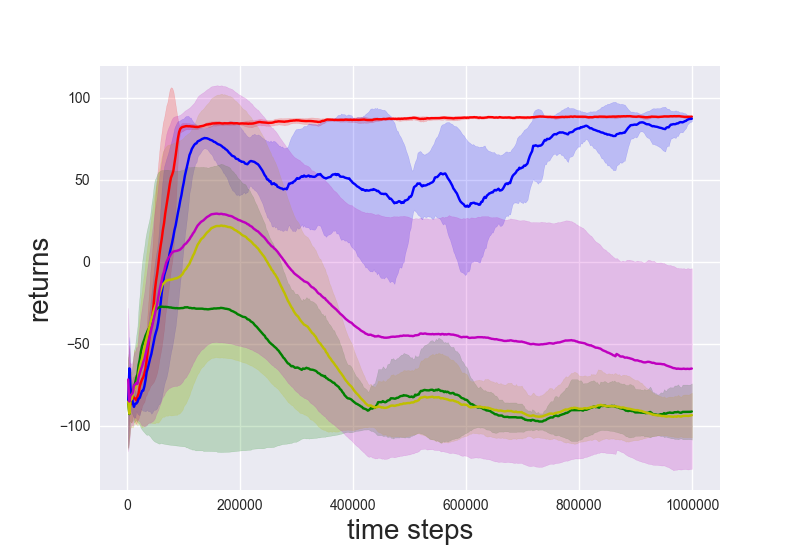}}
\subfigure[\textbf{Inverted Pendulum}]{\includegraphics[width=.4\linewidth]{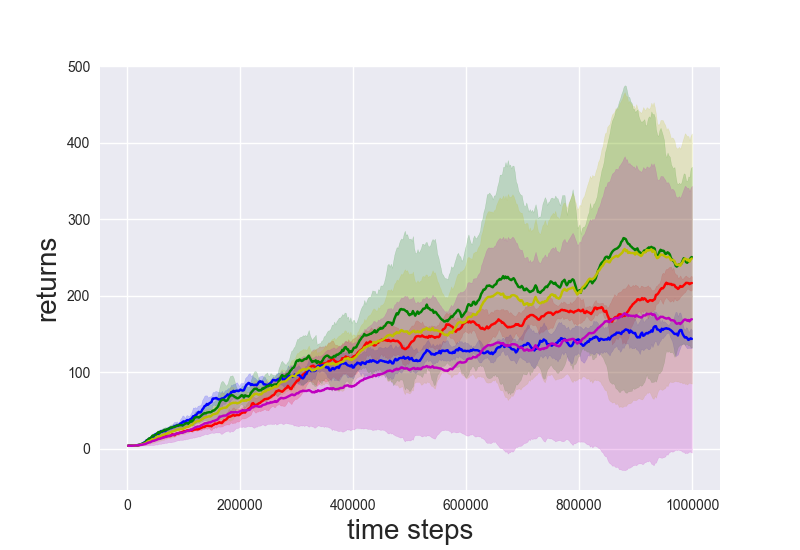}}
\subfigure[\textbf{Acrobot Swingup }]{\includegraphics[width=.4\linewidth]{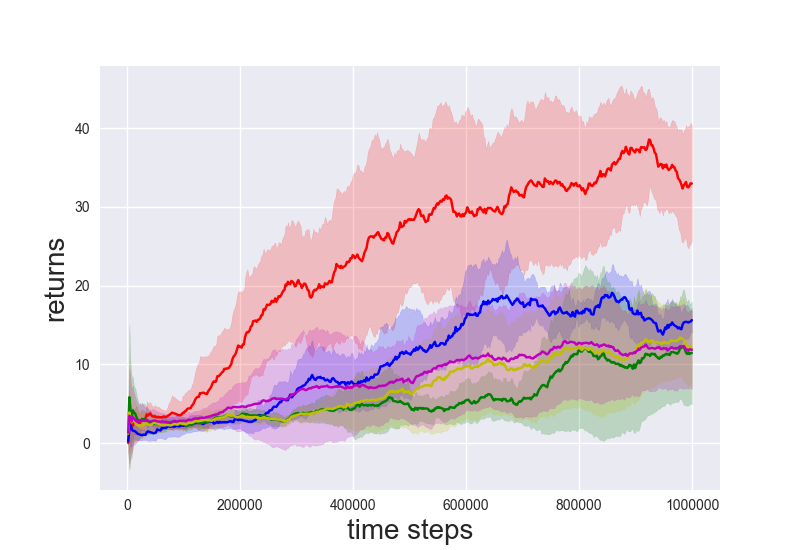}}
\subfigure[\textbf{Pendulum Swingup }]{\includegraphics[width=.4\linewidth]{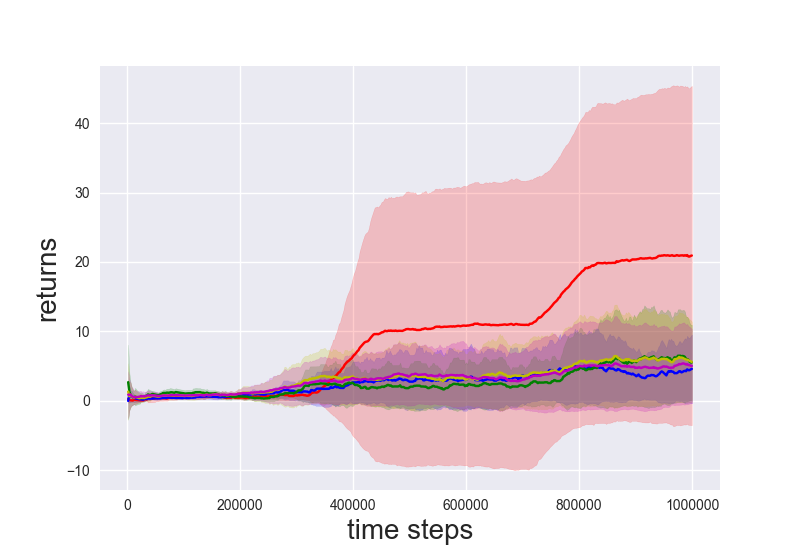}}
\subfigure[\textbf{CartPole v2 }]{\includegraphics[width=.4\linewidth]{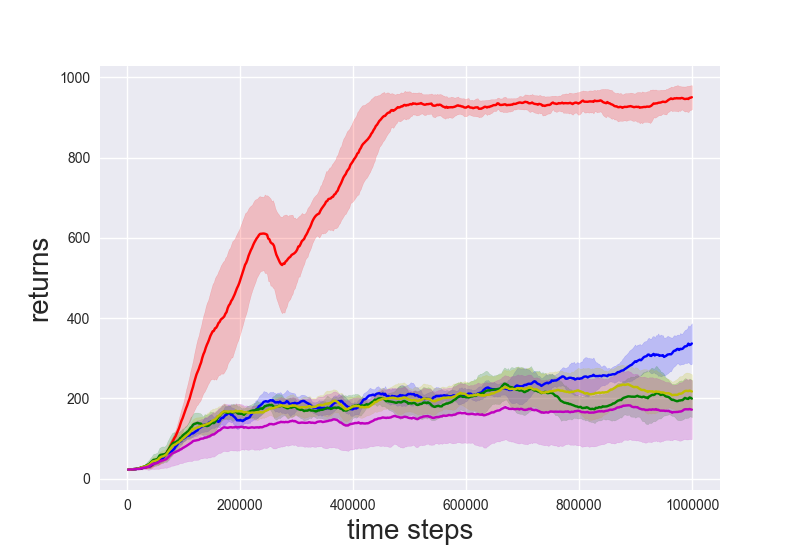}}
\subfigure[\textbf{CartPole v3 }]{\includegraphics[width=.4\linewidth]{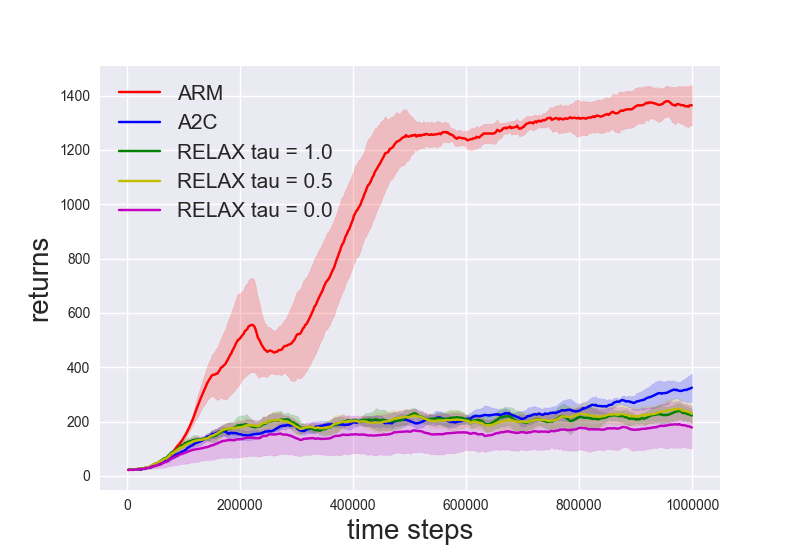}}
\caption{\small{Comparison of $\{\text{ARM},\text{A2C},\text{RELAX}\}$ gradient estimators for on-policy optimization, with A2C advantage estimation: Here RELAX gradient estimator is the generalized version with a combination coefficient $\tau \in [0,1]$ such that $\tau = 1$ recovers the original RELAX estimator while $\tau = 0$ recovers the A2C estimator. We observe that ARM gradient estimator consistently outperforms all gradient estimator baselines, both in terms of convergence rate and asymptotic performance.}}
\label{figure:a2c}
\end{figure}

In Figure \ref{figure:a2c}, we show the performance of $\{\text{ARM},\text{A2C}\}$ estimators along with RELAX with varying $\tau \in \{0, \frac{1}{2}, 1\}$. We see that the ARM policy gradient estimator outperforms the other baseline estimators on most tasks: except on Inverted Pendulum, where all estimators tend to learn slowly, while RELAX with $\tau \in \{\frac{1}{2},1\}$ perform the best in terms of mean performance, they do not significantly outperform others when accounting for the standard deviation. For other benchmark tasks, the ARM policy gradient estimator enables significantly faster policy optimization, while other baselines either become very unstable or learn very slowly.

\paragraph{Generalized Advantage Estimator (GAE).} GAE constructs advantage estimates using a more complex combination of Monte Carlo samples and baseline function, to better trade-off bias and variance. By construction, here the baseline function must be value function critic, hence we only evaluate $\{\text{ARM},\text{A2C}\}$.

In Figure \ref{figure:gae}, we show the comparison. We make several observations: (1) Comparing A2C with GAE, as shown in Figure \ref{figure:gae}, to A2C with A2C advantage estimator, as shown in Figure \ref{figure:a2c}, we see that in most cases A2C with GAE speeds up the policy optimization significantly. This demonstrates the importance of stable advantage estimation for policy gradient estimator. A notable exception is MountainCar, where A2C with A2C advantage estimator performs better. (2) Comparing ARM with A2C in Figure \ref{figure:gae}, we still see that the ARM estimator significantly outperforms A2C estimator. For almost all presented tasks, the ARM policy gradient estimator allows for much faster learning and significantly better asymptotic performance.

\begin{figure}[!t]
\centering
\subfigure[\textbf{MountainCar }]{\includegraphics[width=.4\linewidth]{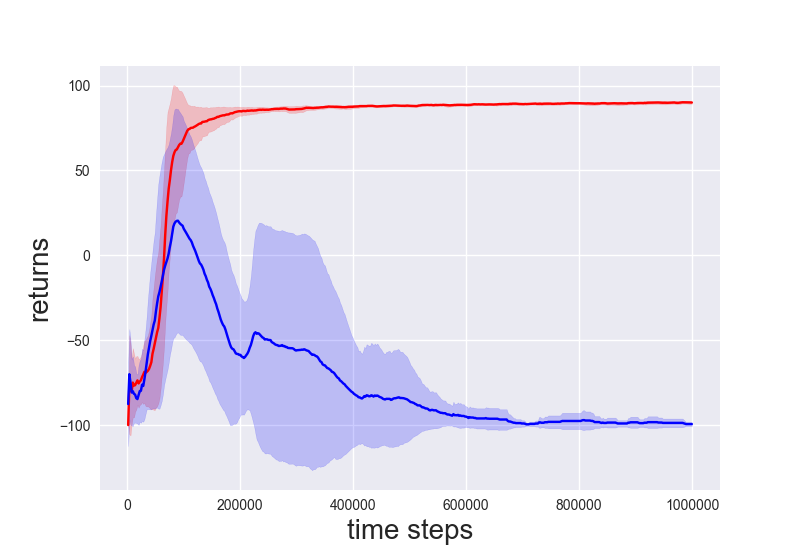}}
\subfigure[\textbf{Inverted Pendulum}]{\includegraphics[width=.4\linewidth]{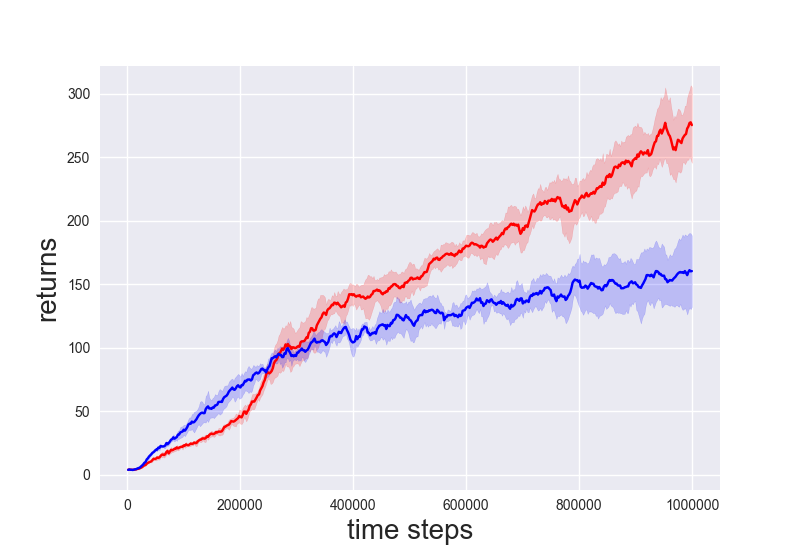}}
\subfigure[\textbf{Acrobot Swingup }]{\includegraphics[width=.4\linewidth]{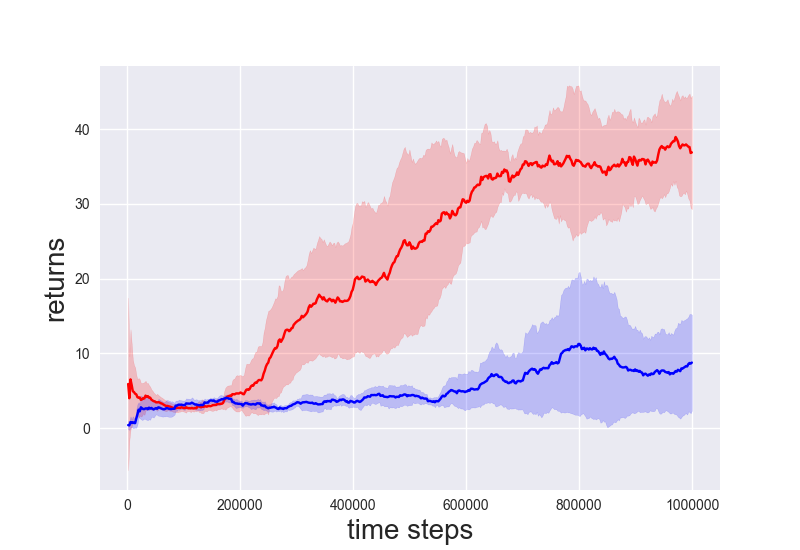}}
\subfigure[\textbf{Pendulum Swingup }]{\includegraphics[width=.4\linewidth]{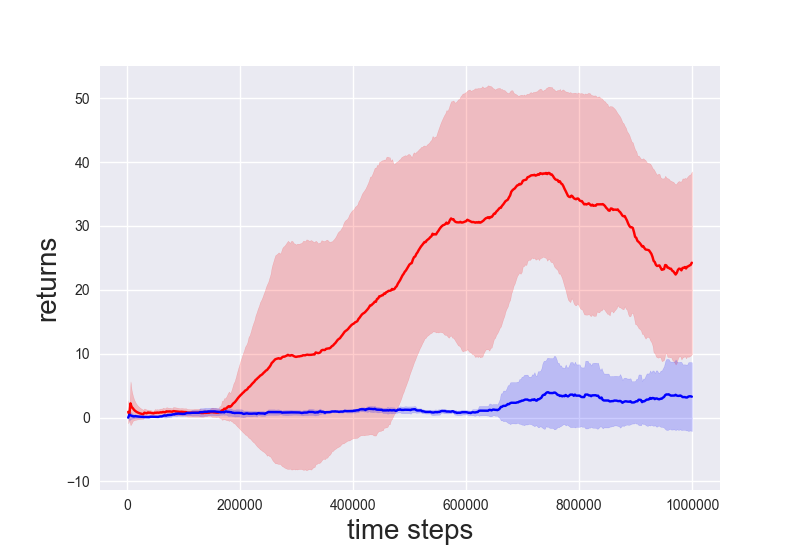}}
\subfigure[\textbf{CartPole v2 }]{\includegraphics[width=.4\linewidth]{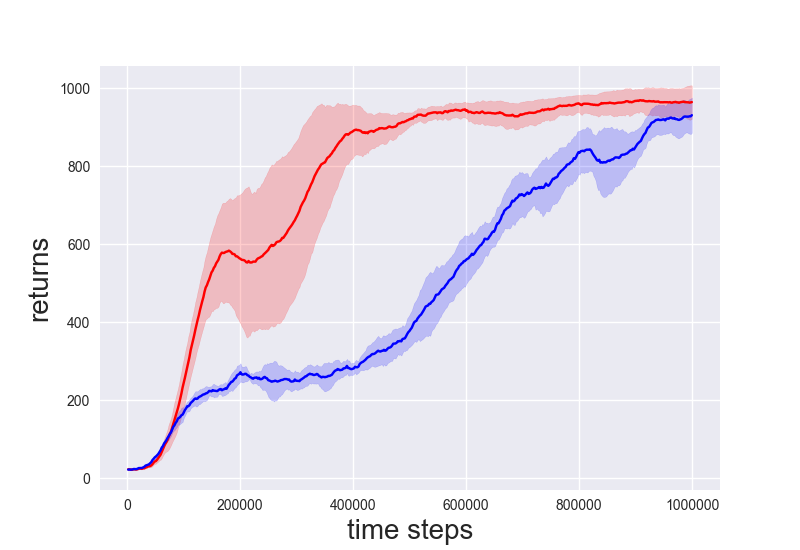}}
\subfigure[\textbf{CartPole v3 }]{\includegraphics[width=.4\linewidth]{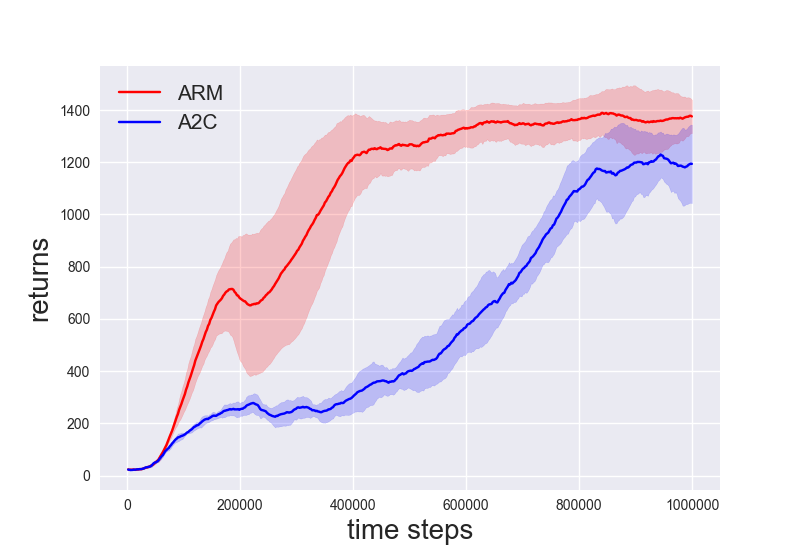}}
\caption{\small{Comparison of $\{\text{ARM},\text{A2C}\}$ gradient estimators for on-policy optimization, with GAE: We observe that ARM gradient estimator still consistently outperforms the A2C gradient estimator. For hard exploration environment MountainCar, ARM leads to fast convergence while A2C does not converge.}}
\label{figure:gae}
\end{figure}

\subsection{Effect of Batch Size}
On-policy gradient estimators are plagued by low sample efficiency, and we typically need many samples to construct high-quality gradient estimator. We vary the batch size $\mathcal{B} \in \{256,512,1024, 2048\}$ for each iteration and evaluate the final performance of $\{\text{ARM},\text{A2C}\}$ gradient estimators. Here, we use GAE as the advantage estimator. We fix the number of iterations (gradient updates) for training to be $\approx 488$ \footnote{This number of iteration is obtained via $\text{int}(\frac{10^6}{2048})$, where $10^6$ is the number of training steps and $2048$ is the default batch size.}. Under this setting, we expect the variance of the policy gradients to increase with decreasing $\mathcal{B}$ and so does the final performance.

\begin{figure}[!t]
\centering
\subfigure[\textbf{CartPole-v0 }]{\includegraphics[width=.35\linewidth]{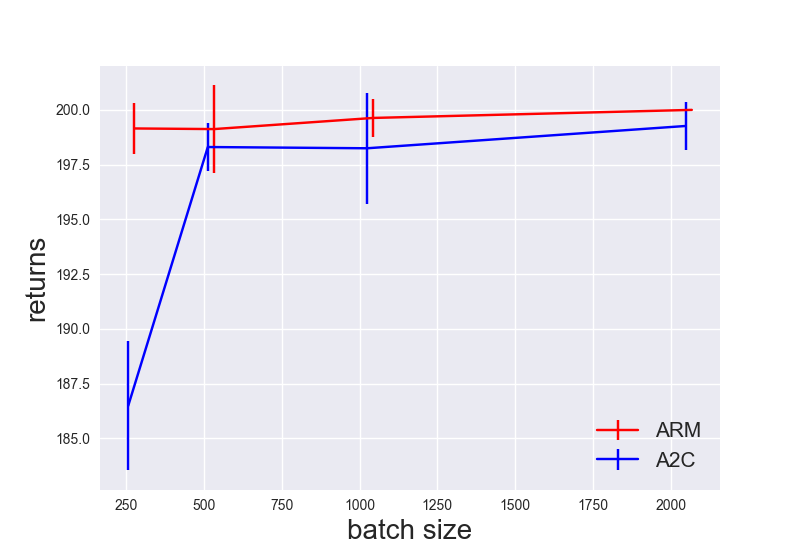}}
%\subfigure[\textbf{CartPole-v2 (T) }]{\includegraphics[width=.4\linewidth]{graph/cartpolev2_batchsize_fixediteration.png}}
\subfigure[\textbf{CartPole-v2 }]{\includegraphics[width=.35\linewidth]{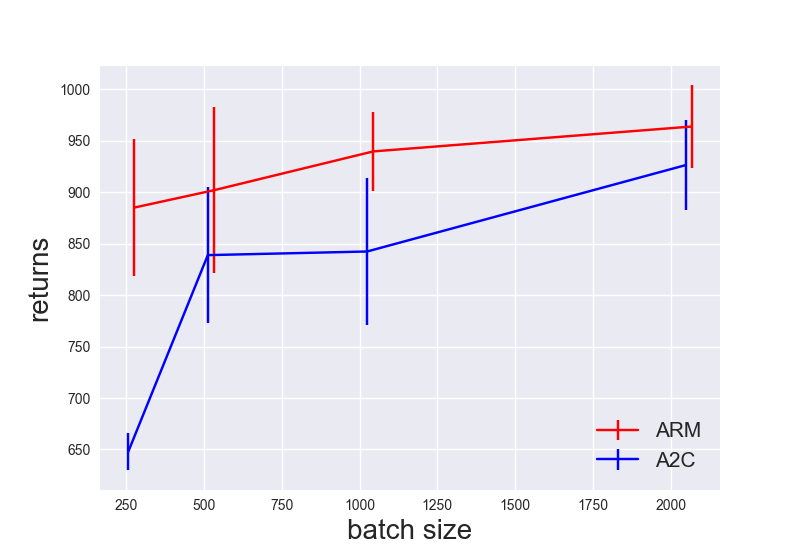}}
\subfigure[\textbf{Pendulum Swingup }]{\includegraphics[width=.35\linewidth]{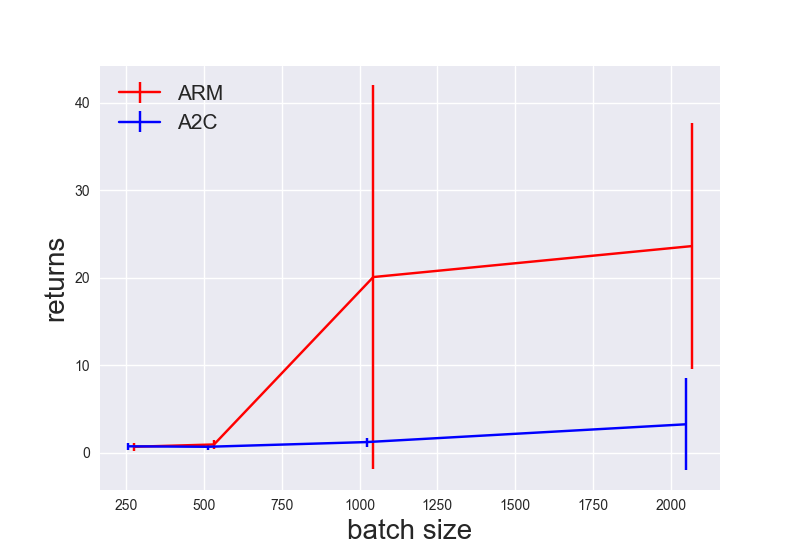}}
%\subfigure[\textbf{Pendulum Swingup (T) }]{\includegraphics[width=.4\linewidth]{graph/dmpendulumswingup_batchsize_fixediteration.png}}
\subfigure[\textbf{Acrobot Swingup }]{\includegraphics[width=.35\linewidth]{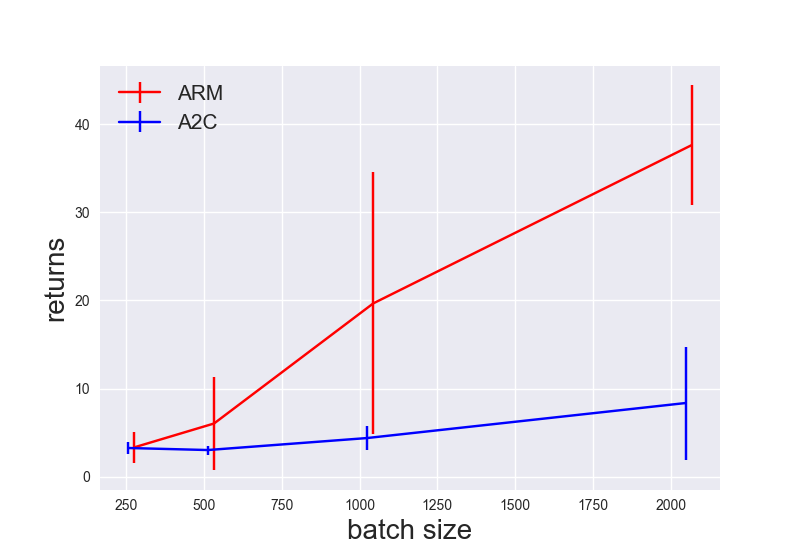}}
%\subfigure[\textbf{Acrobot Swingup (T) }]{\includegraphics[width=.4\linewidth]{graph/dmacrobotswingup_batchsize_fixediteration.png}}
\caption{\small{Comparison of $\{\text{ARM},\text{A2C}\}$ gradient estimators under various batch size $\mathcal{B}$: the number of iterations (gradient updates) are fixed. The final performance of the trained policy increases as the batch size decreases. We observe that across all presented tasks, the performance of ARM dominates that of A2C under all batch sizes. The x-coordinates of the ARM curve are slightly misaligned to separate from the A2C curve.}}
\label{figure:batchsize}
\end{figure}

In Figure \ref{figure:batchsize}, we show the performance of policies trained via $\{\text{ARM},\text{A2C}\}$ gradient estimators. The x-axis show the batch size $\mathcal{B}$ while y-axis show the cumulative returns of the last 10 training iterations (with standard deviation across 5 seeds). We see that as expected the final performance generally increases as we have larger batch size $\mathcal{B}$. Across all presented tasks, the performance of the ARM gradient estimator significantly dominates that of the A2C gradient estimator. This is also compatible with results in Figures \ref{figure:a2c} and  \ref{figure:gae}, where the ARM policy gradient estimator achieves convergence with significantly fewer number of training iterations.

\subsection{Sensitivity to Hyper-parameters}

\begin{figure}[!t]
\centering
\subfigure[\textbf{CartPole v0 }]{\includegraphics[width=.3\linewidth]{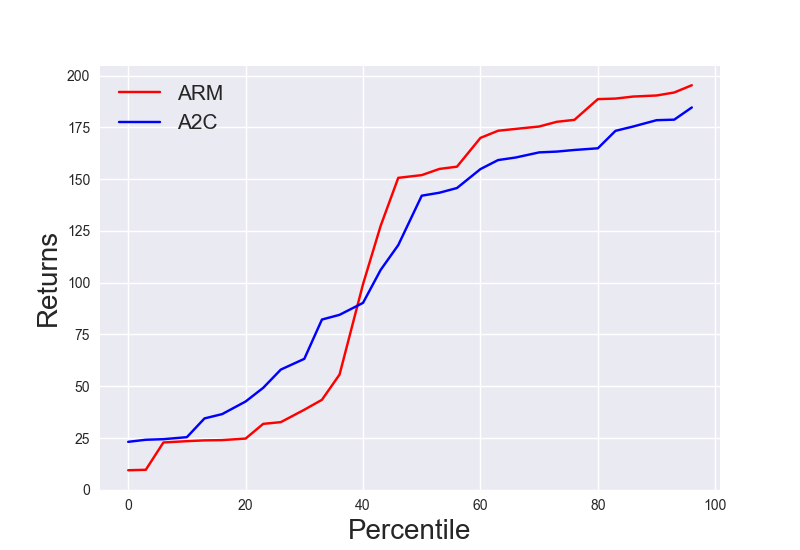}}
\subfigure[\textbf{CartPole v3 }]{\includegraphics[width=.3\linewidth]{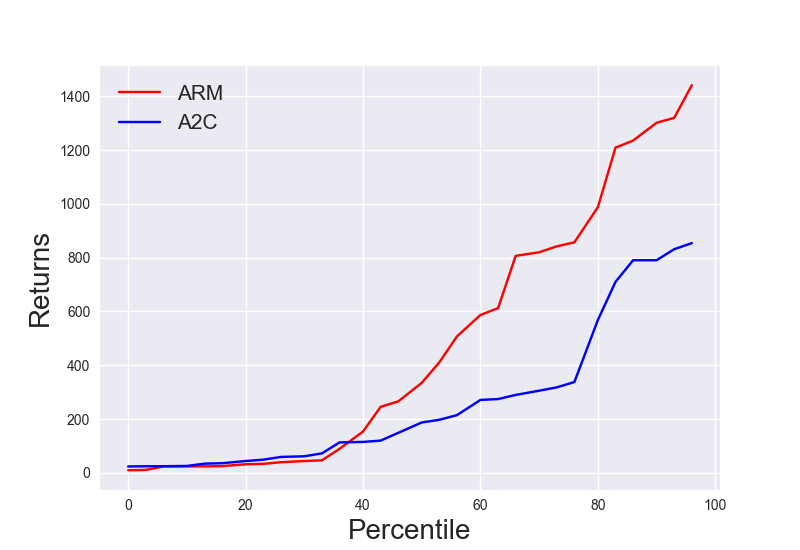}}
\subfigure[\textbf{MountainCar }]{\includegraphics[width=.3\linewidth]{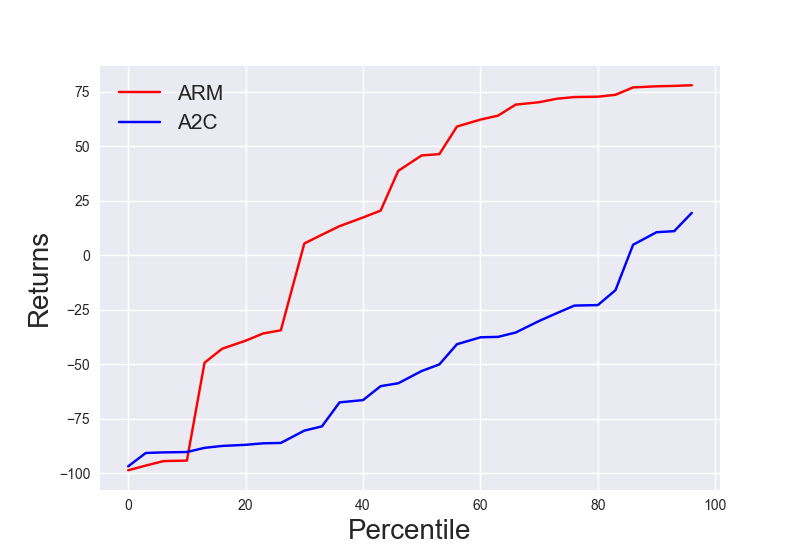}}
\subfigure[\textbf{Inverted Pendulum }]{\includegraphics[width=.3\linewidth]{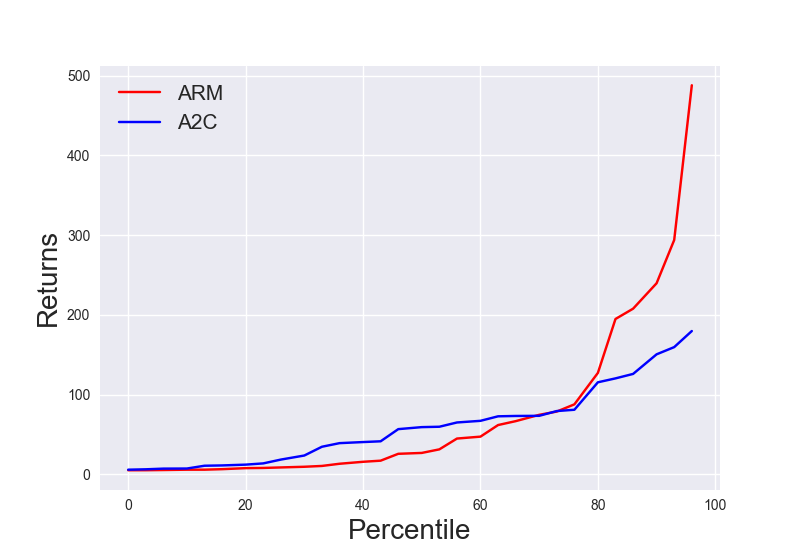}}
\subfigure[\textbf{Acrobot Swingup }]{\includegraphics[width=.3\linewidth]{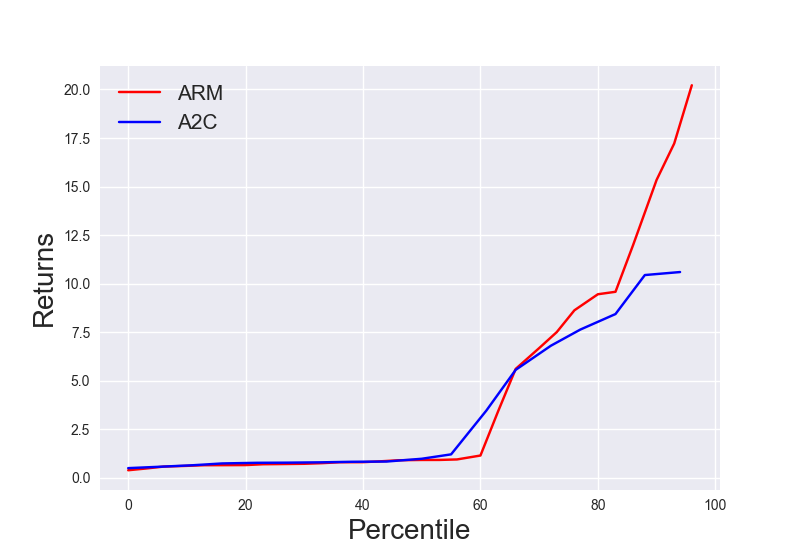}}
\subfigure[\textbf{Pendulum Swingup }]{\includegraphics[width=.3\linewidth]{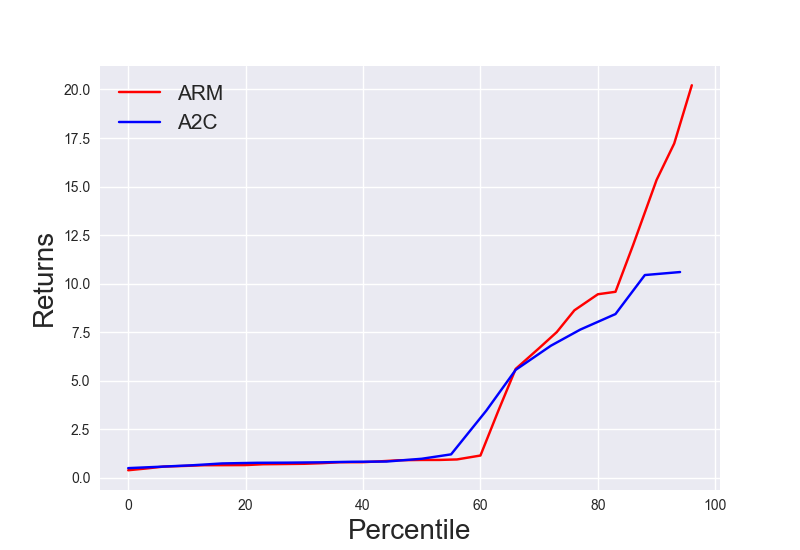}}
\caption{\small{Comparison of $\{\text{ARM},\text{A2C}\}$ gradient estimators under various learning rate and random initializations: We show the quantile plots of $30$ distinct hyper-parameter configurations. ARM gradient estimator is generally more robust to A2C gradient estimator across presented tasks.}}
\label{figure:quantile}
\end{figure}

In addition to batch size, we evaluate the policy gradient estimators' sensitivity to hyper-parameters such as learning rate and random initialization of parameters. In the following setting, we uniformly at random sample log learning rate $\log \alpha \in [-3,-6]$ and parameter initialization from $5$ settings ($5$ random seeds). We train policies under each hyper-parameter configuration for $10^6$ steps and record the performance of the last $50$ iterations. For each policy gradient estimator in $\{\text{ARM},\text{A2C}\}$, we sample $30$ distinct hyper-parameter configurations and plot the quantile plots of their performance in Figure \ref{figure:quantile}. In general, we see that ARM gradient estimator is more robust than A2C gradient estimator across all presented tasks.

\section{Conclusion}
We propose the ARM policy gradient estimator as a convenient low-variance plug-in alternative to prior baseline on-policy gradient estimators, with a simple on-policy optimization algorithm for tasks with binary action space. We leave the extension to more general discrete  action space as exciting future work.

%\newpage
\bibliography{ARMPG_v2}

\begin{thebibliography}{31}
\providecommand{\natexlab}[1]{#1}
\providecommand{\url}[1]{\texttt{#1}}
\expandafter\ifx\csname urlstyle\endcsname\relax
  \providecommand{\doi}[1]{doi: #1}\else
  \providecommand{\doi}{doi: \begingroup \urlstyle{rm}\Url}\fi

\bibitem[Abadi et~al.(2016)Abadi, Barham, Chen, Chen, Davis, Dean, Devin,
  Ghemawat, Irving, Isard, et~al.]{abadi2016tensorflow}
Mart{\'\i}n Abadi, Paul Barham, Jianmin Chen, Zhifeng Chen, Andy Davis, Jeffrey
  Dean, Matthieu Devin, Sanjay Ghemawat, Geoffrey Irving, Michael Isard, et~al.
\newblock Tensorflow: a system for large-scale machine learning.
\newblock In \emph{OSDI}, volume~16, pp.\  265--283, 2016.

\bibitem[Blei et~al.(2017)Blei, Kucukelbir, and McAuliffe]{blei2017}
David~M Blei, Alp Kucukelbir, and Jon~D McAuliffe.
\newblock Variational inference: A review for statisticians.
\newblock \emph{Journal of the American Statistical Association}, 112\penalty0
  (518):\penalty0 859--877, 2017.

\bibitem[Brockman et~al.(2016)Brockman, Cheung, Pettersson, Schneider,
  Schulman, Tang, and Zaremba]{brockman2016}
Greg Brockman, Vicki Cheung, Ludwig Pettersson, Jonas Schneider, John Schulman,
  Jie Tang, and Wojciech Zaremba.
\newblock Openai gym.
\newblock \emph{arXiv preprint arXiv:1606.01540}, 2016.

\bibitem[Dhariwal et~al.(2017)Dhariwal, Hesse, Klimov, Nichol, Plappert,
  Radford, Schulman, Sidor, and Wu]{baselines}
Prafulla Dhariwal, Christopher Hesse, Oleg Klimov, Alex Nichol, Matthias
  Plappert, Alec Radford, John Schulman, Szymon Sidor, and Yuhuai Wu.
\newblock Openai baselines.
\newblock \url{https://github.com/openai/baselines}, 2017.

\bibitem[Fujita \& Maeda(2018)Fujita and Maeda]{fujita2018clipped}
Yasuhiro Fujita and Shin-ichi Maeda.
\newblock Clipped action policy gradient.
\newblock \emph{arXiv preprint arXiv:1802.07564}, 2018.

\bibitem[Grathwohl et~al.(2017)Grathwohl, Choi, Wu, Roeder, and
  Duvenaud]{grathwohl2017backpropagation}
Will Grathwohl, Dami Choi, Yuhuai Wu, Geoff Roeder, and David Duvenaud.
\newblock Backpropagation through the void: Optimizing control variates for
  black-box gradient estimation.
\newblock \emph{arXiv preprint arXiv:1711.00123}, 2017.

\bibitem[Gu et~al.(2015)Gu, Levine, Sutskever, and Mnih]{gu2015muprop}
Shixiang Gu, Sergey Levine, Ilya Sutskever, and Andriy Mnih.
\newblock Muprop: Unbiased backpropagation for stochastic neural networks.
\newblock \emph{arXiv preprint arXiv:1511.05176}, 2015.

\bibitem[Gu et~al.(2017)Gu, Lillicrap, Turner, Ghahramani, Sch{\"o}lkopf, and
  Levine]{gu2017interpolated}
Shixiang~Shane Gu, Timothy Lillicrap, Richard~E Turner, Zoubin Ghahramani,
  Bernhard Sch{\"o}lkopf, and Sergey Levine.
\newblock Interpolated policy gradient: Merging on-policy and off-policy
  gradient estimation for deep reinforcement learning.
\newblock In \emph{Advances in Neural Information Processing Systems}, pp.\
  3846--3855, 2017.

\bibitem[Jang et~al.(2016)Jang, Gu, and Poole]{jang2016categorical}
Eric Jang, Shixiang Gu, and Ben Poole.
\newblock Categorical reparameterization with gumbel-softmax.
\newblock \emph{arXiv preprint arXiv:1611.01144}, 2016.

\bibitem[Kingma \& Welling(2013)Kingma and Welling]{kingma2013}
Diederik~P Kingma and Max Welling.
\newblock Auto-encoding variational bayes.
\newblock \emph{arXiv preprint arXiv:1312.6114}, 2013.

\bibitem[Kucukelbir et~al.(2017)Kucukelbir, Tran, Ranganath, Gelman, and
  Blei]{kucukelbir2016automatic}
Alp Kucukelbir, Dustin Tran, Rajesh Ranganath, Andrew Gelman, and David~M.
  Blei.
\newblock Automatic differentiation variational inference.
\newblock \emph{Journal of Machine Learning Research, 18(14):1-45}, 2017.

\bibitem[Levine et~al.(2016)Levine, Finn, Darrell, and Abbeel]{levine2016}
Sergey Levine, Chelsea Finn, Trevor Darrell, and Pieter Abbeel.
\newblock End-to-end training of deep visuomotor policies.
\newblock \emph{The Journal of Machine Learning Research}, 17\penalty0
  (1):\penalty0 1334--1373, 2016.

\bibitem[Liu et~al.(2018)Liu, Feng, Mao, Zhou, Peng, and Liu]{liu2018action}
Hao Liu, Yihao Feng, Yi~Mao, Dengyong Zhou, Jian Peng, and Qiang Liu.
\newblock Action-dependent control variates for policy optimization via stein
  identity.
\newblock 2018.

\bibitem[Maddison et~al.(2016)Maddison, Mnih, and Teh]{maddison2016concrete}
Chris~J Maddison, Andriy Mnih, and Yee~Whye Teh.
\newblock The concrete distribution: A continuous relaxation of discrete random
  variables.
\newblock \emph{arXiv preprint arXiv:1611.00712}, 2016.

\bibitem[Mnih et~al.(2013)Mnih, Kavukcuoglu, Silver, Graves, Antonoglou,
  Wierstra, and Riedmiller]{mnih2013}
Volodymyr Mnih, Koray Kavukcuoglu, David Silver, Alex Graves, Ioannis
  Antonoglou, Daan Wierstra, and Martin Riedmiller.
\newblock Playing atari with deep reinforcement learning.
\newblock \emph{arXiv preprint arXiv:1312.5602}, 2013.

\bibitem[Mnih et~al.(2016)Mnih, Badia, Mirza, Graves, Lillicrap, Harley,
  Silver, and Kavukcuoglu]{mnih2016}
Volodymyr Mnih, Adria~Puigdomenech Badia, Mehdi Mirza, Alex Graves, Timothy
  Lillicrap, Tim Harley, David Silver, and Koray Kavukcuoglu.
\newblock Asynchronous methods for deep reinforcement learning.
\newblock In \emph{International Conference on Machine Learning}, pp.\
  1928--1937, 2016.

\bibitem[Paisley et~al.(2012)Paisley, Blei, and Jordan]{paisley2012variational}
John Paisley, David Blei, and Michael Jordan.
\newblock Variational bayesian inference with stochastic search.
\newblock \emph{arXiv preprint arXiv:1206.6430}, 2012.

\bibitem[Ranganath et~al.(2014)Ranganath, Gerrish, and
  Blei]{ranganath2014black}
Rajesh Ranganath, Sean Gerrish, and David Blei.
\newblock Black box variational inference.
\newblock In \emph{Artificial Intelligence and Statistics}, pp.\  814--822,
  2014.

\bibitem[Schulman et~al.(2015{\natexlab{a}})Schulman, Levine, Abbeel, Jordan,
  and Moritz]{schulman2015}
John Schulman, Sergey Levine, Pieter Abbeel, Michael Jordan, and Philipp
  Moritz.
\newblock Trust region policy optimization.
\newblock In \emph{International Conference on Machine Learning}, pp.\
  1889--1897, 2015{\natexlab{a}}.

\bibitem[Schulman et~al.(2015{\natexlab{b}})Schulman, Moritz, Levine, Jordan,
  and Abbeel]{schulman2015high}
John Schulman, Philipp Moritz, Sergey Levine, Michael Jordan, and Pieter
  Abbeel.
\newblock High-dimensional continuous control using generalized advantage
  estimation.
\newblock \emph{arXiv preprint arXiv:1506.02438}, 2015{\natexlab{b}}.

\bibitem[Schulman et~al.(2017{\natexlab{a}})Schulman, Wolski, Dhariwal,
  Radford, and Klimov]{schulman2017}
John Schulman, Filip Wolski, Prafulla Dhariwal, Alec Radford, and Oleg Klimov.
\newblock Proximal policy optimization algorithms.
\newblock \emph{arXiv preprint arXiv:1707.06347}, 2017{\natexlab{a}}.

\bibitem[Schulman et~al.(2017{\natexlab{b}})Schulman, Wolski, Dhariwal,
  Radford, and Klimov]{schulman2017proximal}
John Schulman, Filip Wolski, Prafulla Dhariwal, Alec Radford, and Oleg Klimov.
\newblock Proximal policy optimization algorithms.
\newblock \emph{arXiv preprint arXiv:1707.06347}, 2017{\natexlab{b}}.

\bibitem[Silver et~al.(2016)Silver, Huang, Maddison, Guez, Sifre, Van
  Den~Driessche, Schrittwieser, Antonoglou, Panneershelvam, Lanctot,
  et~al.]{silver2016}
David Silver, Aja Huang, Chris~J Maddison, Arthur Guez, Laurent Sifre, George
  Van Den~Driessche, Julian Schrittwieser, Ioannis Antonoglou, Veda
  Panneershelvam, Marc Lanctot, et~al.
\newblock Mastering the game of go with deep neural networks and tree search.
\newblock \emph{nature}, 529\penalty0 (7587):\penalty0 484--489, 2016.

\bibitem[Sutton et~al.(2000)Sutton, McAllester, Singh, and Mansour]{sutton1999}
Richard~S Sutton, David~A McAllester, Satinder~P Singh, and Yishay Mansour.
\newblock Policy gradient methods for reinforcement learning with function
  approximation.
\newblock In \emph{Advances in neural information processing systems}, pp.\
  1057--1063, 2000.

\bibitem[Tassa et~al.(2018)Tassa, Doron, Muldal, Erez, Li, Casas, Budden,
  Abdolmaleki, Merel, Lefrancq, et~al.]{tassa2018deepmind}
Yuval Tassa, Yotam Doron, Alistair Muldal, Tom Erez, Yazhe Li, Diego de~Las
  Casas, David Budden, Abbas Abdolmaleki, Josh Merel, Andrew Lefrancq, et~al.
\newblock Deepmind control suite.
\newblock \emph{arXiv preprint arXiv:1801.00690}, 2018.

\bibitem[Todorov(2008)]{todorov2008}
Emanuel Todorov.
\newblock General duality between optimal control and estimation.
\newblock In \emph{Decision and Control, 2008. CDC 2008. 47th IEEE Conference
  on}, pp.\  4286--4292. IEEE, 2008.

\bibitem[Tucker et~al.(2017)Tucker, Mnih, Maddison, Lawson, and
  Sohl-Dickstein]{tucker2017rebar}
George Tucker, Andriy Mnih, Chris~J Maddison, John Lawson, and Jascha
  Sohl-Dickstein.
\newblock Rebar: Low-variance, unbiased gradient estimates for discrete latent
  variable models.
\newblock In \emph{Advances in Neural Information Processing Systems}, pp.\
  2627--2636, 2017.

\bibitem[Tucker et~al.(2018)Tucker, Bhupatiraju, Gu, Turner, Ghahramani, and
  Levine]{tucker2018mirage}
George Tucker, Surya Bhupatiraju, Shixiang Gu, Richard~E Turner, Zoubin
  Ghahramani, and Sergey Levine.
\newblock The mirage of action-dependent baselines in reinforcement learning.
\newblock \emph{arXiv preprint arXiv:1802.10031}, 2018.

\bibitem[Williams(1992)]{williams1992}
Ronald~J Williams.
\newblock Simple statistical gradient-following algorithms for connectionist
  reinforcement learning.
\newblock In \emph{Reinforcement Learning}, pp.\  5--32. Springer, 1992.

\bibitem[Wu et~al.(2018)Wu, Rajeswaran, Duan, Kumar, Bayen, Kakade, Mordatch,
  and Abbeel]{wu2018variance}
Cathy Wu, Aravind Rajeswaran, Yan Duan, Vikash Kumar, Alexandre~M Bayen, Sham
  Kakade, Igor Mordatch, and Pieter Abbeel.
\newblock Variance reduction for policy gradient with action-dependent
  factorized baselines.
\newblock \emph{arXiv preprint arXiv:1803.07246}, 2018.

\bibitem[Yin \& Zhou(2019)Yin and Zhou]{yin2018arm}
Mingzhang Yin and Mingyuan Zhou.
\newblock {ARM}: Augment-{REINFORCE}-merge gradient for stochastic binary
  networks.
\newblock In \emph{International Conference on Learning Representations}, 2019.

\end{thebibliography}
\bibliographystyle{iclr2019_conference}

\appendix
\onecolumn

\section{Further Experiment Details}
\subsection{CartPole Environment Setup}
The CartPole experiments are defined by an environment parameter $T$ which specifies that the agent can achieve a maximum rewards of $T$ ($i.e.$,  balance the system for $T$ time steps before the episode terminates).

In our main experiments, we set $T=200$ for CartPole-v0, $T=500$ for CartPole-v1, $T=1000$ for CartPole-v2 and $T=1500$ for CartPole-v3. The difficulty increases with $T$: with large $T$, the agent is less likely to obtain many full trajectories within a single iteration, making it more difficult for return based estimation; long horizons also make it hard for policy optimization.

\subsection{Advantage Estimator}
We consider two popular advantage estimators widely in use \citep{mnih2016,schulman2015high} for on-policy optimization applications. The objective of both estimators are to approximate the advantage function $A^\pi(s,a)$ under current policy $\pi$.

\paragraph{A2C Advantage Estimator.} We construct the A2C estimator at time $t$ with the following
\begin{align}
    \hat{A}^\pi(s_t,a_t) = \hat{Q}^\pi(s_t,a_t) - V_\phi(s_t),
    \label{eq:a2cestimator}
\end{align}
where $\hat{Q}^\pi(s_t,a_t) = \sum_{l\geq 0} r_{t+l} \gamma^l$ are Monte-Carlo estimates of the partial sum of returns (along the sampled trajectories). The critic $V_\phi(s)$ is trained by regression over the partial returns to approximate the value function $V_\phi(s) \approx V^\pi(s)$.

\paragraph{Generalized Advantage Estimator (GAE).} GAE is indexed by two parameters $\gamma$ and $\lambda$, where $\gamma\in(0,1)$ is the discount factor and $\lambda\in(0,1]$ is an additional trace parameter that determines bias-variance trade-off of the final estimation. We define TD-errors 
\begin{align}
    \delta_t^V = r_t + \gamma V_\phi(s_{t+1}) - V_\phi(s_t),
\end{align}
where $V_\phi(s) \approx V^\pi(s)$ is a value function critic trained by regression over returns. GAE at time $t$ is computed as a weighted average of TD-errors across time
\begin{align}
    \hat{A}^{\pi}(s_t,a_t) = \sum_{l=0}^\infty (\gamma\lambda)^l \delta_{t+l}^V.
    \label{eq:gaeestimator}
\end{align}
Though the optimal $\lambda$ parameter is problem dependent, a common practice is to set $\lambda = 0.95$. 

\subsection{Baseline Implementations}
We have compared the ARM policy gradient estimator with A2C policy gradient estimator and recently proposed RELAX gradient estimator. We implement three policy gradient estimators based on OpenAI baselines \citep{baselines}. Though each policy gradient estimator requires more or less different implementation variations ($e.g.$,  record the pseudo action $a_t^{(s)}$ and random noise $u_t$ for the ARM policy gradient), we have ensured that these three implementations share as much common structure as possible.

We note that though the RELAX code \citep{grathwohl2017backpropagation} is made available, and their code is built on top of the OpenAI baselines. We did not directly run their code because of potential issues in their original implementation of RELAX. We implement our own version of RELAX and note some of the differences from their code.

\paragraph{Difference 1: Policy gradient computation.}  In the following we point out potential issues with the implementation of \citep{grathwohl2017backpropagation} and we refer to the the latest commit to the RELAX repository\footnote{\url{https://github.com/wgrathwohl/BackpropThroughTheVoidRL}} (commit 0e6623d) as of this writing. 

Recall that on-policy optimization algorithms alternate between performing rollouts and performing updates based on the rollout samples. At rollout time, $\tau$ on-policy samples $\{s_t,a_t,r_t\}_{t=0}^{\tau-1}$ are collected. Recall that the A2C policy gradient estimator takes the following form
\begin{align}
\hat{g}_\theta^{\text{A2C}} = \frac{1}{\tau} \sum_{t=0}^{\tau-1} A_t \nabla_\theta \log \pi_\theta(a_t|s_t), \label{eq:a2csample}
\end{align}
where $A_t$ are the advantage estimators from the on-policy samples. Importantly in (\ref{eq:a2csample}), the actions $a_t$ should be the on-policy samples - the intuition is that on-sample actions $a_t$ match their corresponding advantages $A_t$, if $A_t > 0$ for a certain action $a_t$ in state $s_t$, the gradient update will increase the probability $\pi_\theta(a_t|s_t)$. In some cases, one implements the A2C gradient estimator by re-sampling actions $a_t^{(r)}$ at each $s_t$ during training, resulting in the following estimator
\begin{align}
\hat{g}_\theta^{\text{A2C(r)}} = \frac{1}{\tau} \sum_{t=0}^{\tau-1} A_t \nabla_\theta \log \pi_\theta(a_t^{(r)}|s_t). \label{eq:a2cresample}
\end{align}
We remark that this new gradient estimator with re-sampled actions $a_t^{(r)}$ is biased, i.e. $\mathbb{E}[\hat{g}_\theta^{\text{A2C(r)}}] \neq \mathbb{E}[\hat{g}_\theta^{\text{A2C}}] = g_\theta$. The bias comes from the fact that when $a_t \neq a_t^{(r)}$, we assign the advantage $A_t$ to the wrong actions $a_t^{(r)}$, leading to mismatched credit assignments. We find the latest version of RELAX code \citep{grathwohl2017backpropagation} implements this biased estimator (\ref{eq:a2cresample}) - in fact they implement the biased estimator for both the A2C baseline and their proposed RELAX gradient estimator. 

Specifically, in Tensorflow terminology \citep{abadi2016tensorflow}, the advantages $A_t$, actions $a_t$ and states $s_t$ should be input into the loss and gradient computation via \emph{placeholders}. However, in the RELAX implementation, they input advantages $A_t$ and states $s_t$ via \emph{placeholders}, while inputting the actions via \emph{train\_model.a0} where \emph{a0} stands for actions sampled from the policy network \emph{train\_model}. In practice, this will cause the policy model to re-sample an independent set of actions, leading to biased estimates. The re-sampling bias is severe when $a_t \neq a_t^{(r)}$, especially when the policy is still random during the initial stage of training. Later in the training when the policy becomes more deterministic, the bias decreases since it is more likely that $a_t = a_t^{(r)}$.  In our implementation, we correct such potential bugs.

\paragraph{Difference 2: Average gradients over states not trajectories.} In the original development of RELAX, policy gradients are computed per trajectories and averaged across multiple trajectories. A common practice in on-policy algorithm implementation \citep{baselines} is to average policy gradients across states. We follow this latter practice. As a result, we can collect a fixed number of steps per iteration instead of a fixed number of rollouts (which can result in varying number of steps) as in the original work \citep{grathwohl2017backpropagation}. We believe that such practice allows for fair comparison.

%##########

\section{Algorithm}
In the pseudocode below, we omit the training of value function baseline $V_\phi(s) \approx V^\pi(s)$. Following the common practice \citep{schulman2015,mnih2016,schulman2017}, the value function baseline is trained by regression over Monte Carlo returns.

\begin{algorithm}[H]
\begin{algorithmic}[1]
\STATE\textbf{Input}: total number of time steps $T$, training batch size $\tau$, learning rate for policy $\alpha$, discount factor $\gamma$ \; \\
\STATE\textbf{Initialize}: Policy network $\pi_\theta(s)$ with parameters $\theta$.

\FOR {$t=1, 2, ... T$} 
		\STATE // \emph{Rollout}
		\STATE
		1. In state $s_t$, compute logit  $\mathcal{T}_\theta(s_t)$. Sample $u_t \sim \mathcal{U}(0,1)$. If $u \leq \sigma(\mathcal{T}_\theta(s_t))$, let the action $a_t = 1$ and otherwise $a_t = 0$; if $u \geq \sigma(-\mathcal{T}_\theta(s_t))$, let the shadow action $a_t^{(s)} = 1$ and otherwise $a_t^{(s)} = 0$. Receive instant reward $r_t$. Label $\pi_t = \sigma(\mathcal{T}_\theta(s_t))$.
		
		\IF{$\text{counter}\  \text{mod}\  \tau == 0$}
		\STATE // \emph{Training policy}
		
		\STATE 2. Estimate advantage function for each step  $\{A_i\}_{i=0}^{\tau-1}$ from the collected $\tau$ sample tuples labeled as $\{s_i,a_i,a_i^{(s)},r_i,u_i,\pi_i\}_{i=0}^{\tau-1}$.

		\STATE 3. Construct differences of advantage functions at each step based on (\ref{eq:adv})
		\begin{align}
		  D_i = (\frac{\pi_i}{1-\pi_i} + 1) (u_i - \frac{1}{2}) A_i \cdot |a_i^{(s)} - a_i| . \nonumber
		\end{align}
		
		\STATE  4. Construct the surrogate loss function
		\begin{align}
		    L_p = -\frac{1}{\tau} \sum_{i=0}^{\tau-1} D_i \mathcal{T}_\theta(s_i).
		\nonumber
		\end{align}	
		\STATE 5. Update $\theta \leftarrow \theta - \alpha \nabla_\theta L_p$.
		\ENDIF	
 \ENDFOR
 \caption{On-Policy Optimization with ARM Policy Gradient}
\label{Alg:arm}
\end{algorithmic}
\end{algorithm}

\section{Further Discussions on ARM Policy Gradient Estimator}

For simplicity we fix a state $s$ and use a simplified notation: $A_1 = A^\pi(s,a=1)$, $A_0 = A^\pi(s,a=0)$, $\pi_1 = \pi(a=1\,|\,s) = \frac{\exp(\phi)}{1+\exp(\phi)}$, and $\pi_0 = 1-\pi_1$. Here the logit $\phi$ is parameterized as $\phi = \phi(\theta)$ by parameter $\theta$. Let $u \sim \mathcal{U}(0,1)$ be the noise used for generating actions such that $a = 1$ if $u < \sigma(\phi)$ and the pseudo action $a^{(s)} = 1$ if  $u > \sigma(-\phi)$. Without loss of generality, assume $\phi > 0$.

The ARM policy gradient estimator can be simplified to be the following
\begin{align}
    g_{\theta}^{\text{ARM}} &= (u - \frac{1}{2}) [(A_1 - A_0) \mathbf{1}_{\{u > \sigma(\phi)\}} %\nonumber \\ &
    + (A_0 - A_1)  \mathbf{1}_{\{u < \sigma(-\phi)\}} ] \nabla_\theta \phi.
    \label{eq:armsimplified}
\end{align}
The intuition for the estimator is clear: when $A_1 > A_0$, the above expression always updates $\phi$ such that $\pi_1$ is increased, whenever $u > \sigma(\phi)$ or $u< \sigma(-\phi)$. Then estimator is exactly zero when $\sigma(-\phi) \leq u \leq \sigma(\phi)$, which will become frequent as the policy becomes less entropic during training.

The A2C gradient estimator is the following
\begin{align}
    g_{\theta}^{\text{A2C}} = [A_1 \pi_0 \mathbf{1}_{\{u \leq \sigma(\phi)\}} - A_0 \pi_1 \mathbf{1}_{\{u > \sigma(\phi)\}} ] \nabla_\theta \phi.     \label{eq:a2csimplified}
\end{align}

%We can construct lower-variance variant of the A2C gradient estimator with antithetic sampling in $u$
%\begin{align}
%    g_{\theta}^{\text{anti-A2C}} &= \frac{1}{2}[(A_1 \pi_0  - A_0 \pi_1) \mathbf{1}_{\{u > \sigma(\phi)\  \text{or}\ u < \sigma(-\phi)\}} \nonumber \\ &+ A_1 \pi_0 \mathbf{1}_{\{\sigma(-\phi) \leq u \leq \sigma(\phi)\}} ] \nabla_\theta \phi.
%    \label{eq:a2cantisimplified}
%\end{align}
The intuition for updates are also clear for (\ref{eq:a2csimplified}): for example when $A_1 > 0$ (which implies $A_0 < 0$), whenever $u \leq \sigma(\phi)$ or $u > \sigma(\phi)$, the gradient always updates the parameter such that $\pi_1$ increases. However, %for both A2C variants, 
due to the A2C gradient form (\ref{eq:a2csimplified}), there are no random noise $u$ that leads to exact zero gradient estimators. This tends to make the updates less stable because even when near a local optima the parameters can still oscillate due to noisy estimates of gradient updates.

We could also compute the expected gradient directly
\begin{align}
    \mathbb{E}[g_{\theta}] = [A_1 - A_0] \pi_1 \pi_0 \nabla_\theta \phi,    \label{eq:expected}
\end{align}
Since computing the expected gradients also requires advantage estimators (by replacing $A_i$ by their estimators $\hat{A}_i$), we call the sample estimate of (\ref{eq:expected})  the expected gradient estimators. Although the expected gradient estimator analytically computes the policy gradient, it suffers a similar issue as the A2C gradient estimator - due to noisy estimates of the advantages, the estimator in (\ref{eq:expected}) can never take on exactly zero values. In practice, as we will show below, when the learning rate is fixed, this causes the policy parameter to oscillate due to noisy estimates and leads to unstable learning.

\subsection{Further Experiments Comparing Expected Gradient Estimators}
Though the expected gradient estimators analytically compute the policy gradient for parameter updates, it can still suffer from noisy estimates of the advantage function. To be concrete, we illustrate such phenomenon in Figure \ref{figure:expected}, where we evaluate gradient estimators $\{\text{ARM,A2C}\}$ along with expected gradient estimators. The advantage estimators are A2C advantage estimators. We choose the learning rate $\alpha \in \{3\cdot10^{-6},1\cdot 10^{-5},3\cdot 10^{-5}\}$ for the expected gradient estimators and $\alpha = 3\cdot 10^{-5}$ for $\{\text{ARM,A2C}\}$. 

For Figure \ref{figure:expected} (a)(c)(d), we show the learning rate $\alpha = 3\cdot 10^{-5}$ result for the expected gradient estimator since it achieves the best performance and is comparable with the ARM policy gradient estimator. In these cases, we see that the achieve slightly faster rate of convergence than others, with comparable asymptotic performance in (a)(c). In (d), the asymptotic performance of the expected gradient estimator suffers a bit.

For Figure \ref{figure:expected} (b)(e)(f), we show the training curves corresponding to all the learning rates of the expected gradient estimators. In these plots, the expected gradient estimators achieve fast learning during the initial stage of training, yet its performance quickly suffers as a result of unstable learning (notice the sudden drops in performance). Turning down the learning rate from $3\cdot 10^{-5}$ to $3\cdot 10^{-6}$ slightly alleviates such issues at the cost of slower convergence. To fully remedy such problems, we speculate that it is necessary to either introduce an annealing scheme in the learning rate so as to avoid the unstable training, or introduce trust-region based methods to stabilize updates \citep{schulman2015,schulman2017proximal}. On the other hand, the ARM policy gradient estimator consistently achieves stable policy optimization without additional efforts of learning rate annealing and trust-region based techniques.

\begin{figure}[!t]
\centering
\subfigure[\textbf{MountainCar }]{\includegraphics[width=.4\linewidth]{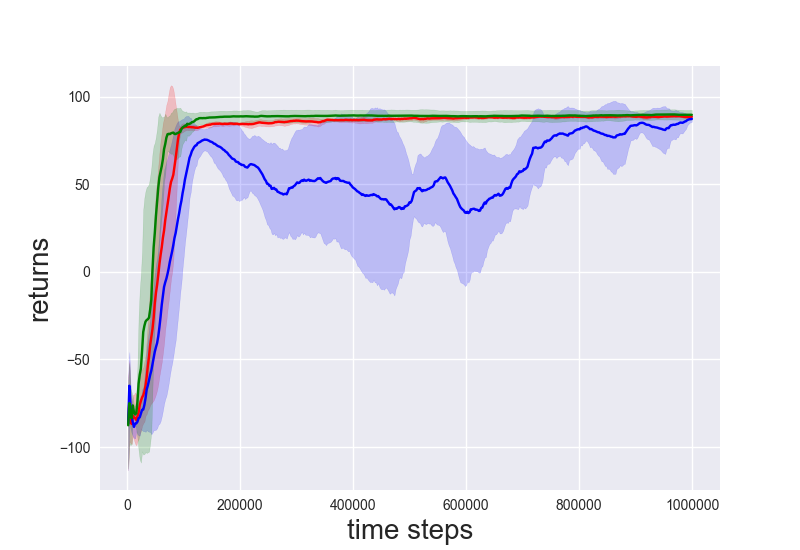}}
\subfigure[\textbf{Inverted Pendulum}]{\includegraphics[width=.4\linewidth]{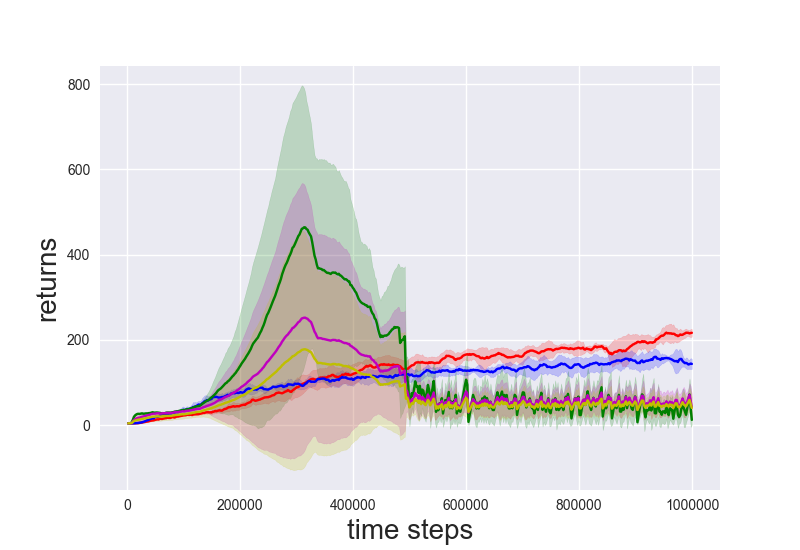}}
\subfigure[\textbf{Acrobot Swingup }]{\includegraphics[width=.4\linewidth]{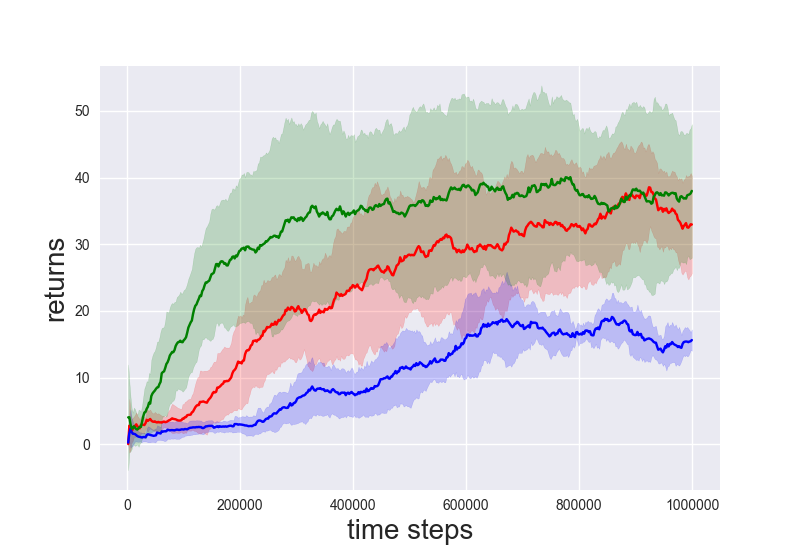}}
\subfigure[\textbf{Pendulum Swingup }]{\includegraphics[width=.4\linewidth]{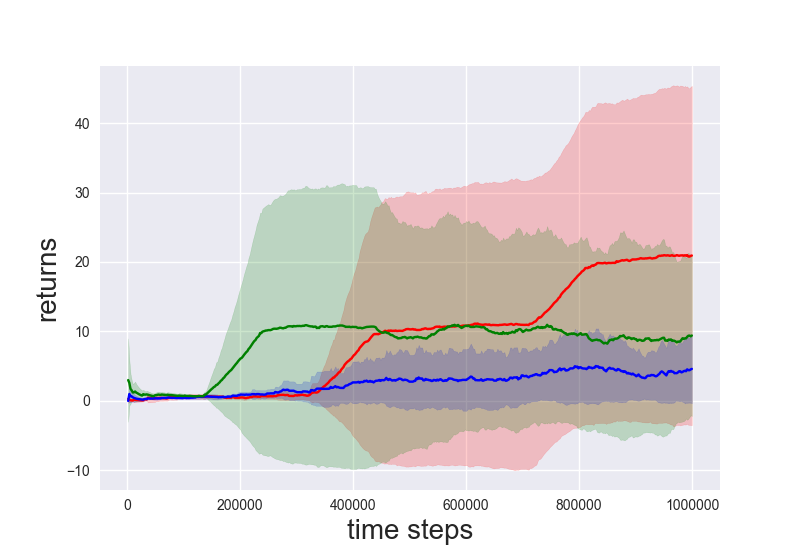}}
\subfigure[\textbf{CartPole v2 }]{\includegraphics[width=.4\linewidth]{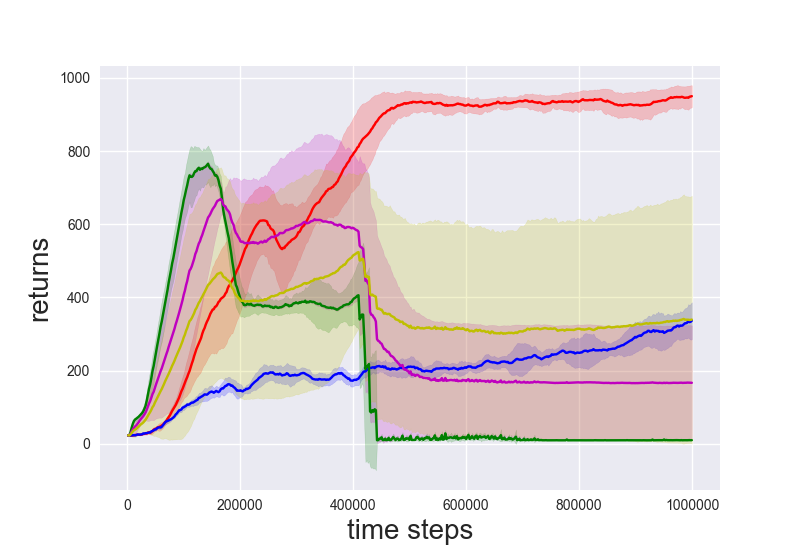}}
\subfigure[\textbf{CartPole v3 }]{\includegraphics[width=.4\linewidth]{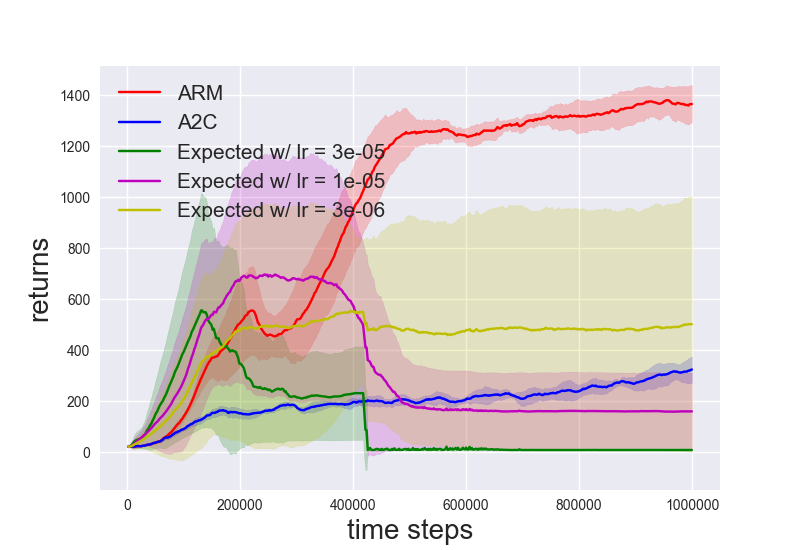}}
\caption{\small{Comparison of $\{\text{ARM},\text{A2C}\}$ gradient estimators and expected gradient estimators with various learning rates for on-policy optimization , with A2C advantage estimator: for the expected gradient estimators we tune the learning rate $\alpha \in \{3\cdot10^{-6},1\cdot 10^{-5},3\cdot 10^{-5}\}$. In (a)(c)(d), We only show the $\alpha = 3\cdot10^{-5}$ result, which achieves the best comparable performance with ARM; in (b)(e)(f), we show curves corresponding to all learning rates since no learning rates can achieve stable optimization - the performance drastically increases initially and quickly descends due to unstable policy optimization.}}
\label{figure:expected}
\end{figure}
\end{document}